\documentclass[11pt]{article}
\usepackage{dsfont} 
\usepackage{times}
\usepackage{latexsym}
\usepackage{bbm} %
\usepackage{bm} %
\usepackage{enumitem}
\usepackage{booktabs}
\usepackage{array, colortbl}
\usepackage{comment}
\usepackage{algorithm}
\usepackage{algpseudocode}
\usepackage{xcolor}
\usepackage[final]{acl}
\usepackage{tikz}
\usetikzlibrary{arrows.meta}
\usetikzlibrary{positioning, fit, backgrounds, calc}
\usepackage{times}
\usepackage{latexsym}
\usetikzlibrary{shadows}
\usepackage{caption}
\usepackage{multirow}
\usepackage{nicefrac}
\usepackage{microtype}
\usepackage{setspace}
\usepackage{tikz}
\usetikzlibrary{decorations.pathreplacing}

\usepackage[T1]{fontenc}

\usepackage[utf8]{inputenc}

\usepackage{microtype}

\usepackage{inconsolata}
\usepackage{amsmath}
\usepackage{amssymb}
\usepackage{amsfonts}
\usepackage{amsthm}

\usepackage{graphicx}

\usepackage{thmtools}
\usepackage{thm-restate}

\usepackage{tcolorbox}
\usepackage{enumitem}
\usepackage{xcolor}

\usepackage{graphicx}
\usepackage[disable]{todonotes}
\definecolor{mintgreen}{RGB}{152, 255, 152}
\definecolor{sepia}{RGB}{94,38,18}
\makeatletter
\newcommand*\iftodonotes{\if@todonotes@disabled\expandafter\@secondoftwo\else\expandafter\@firstoftwo\fi}  %
\makeatother

\newcommand{\stepnote}[1]{\llap{{\color{gray}\scriptsize\itshape (#1)}\quad}}

\makeatletter
\newcommand{\appendixtableofcontents}{%
  \section*{Appendix Contents}%
  \@starttoc{atoc}%
}
\newcommand{\startappendixtoc}{%
  \let\oldaddcontentsline\addcontentsline
  \renewcommand{\addcontentsline}[3]{%
    \oldaddcontentsline{##1}{##2}{##3}%
    \ifstrequal{##1}{toc}{\oldaddcontentsline{atoc}{##2}{##3}}{}%
  }%
}
\makeatother

\usepackage{xcolor}

\usepackage{tikz}
\usepackage{mleftright}
\usetikzlibrary{positioning,arrows.meta}
\newcommand{\mymacro}[1]{%
  \ifmmode
    {\textcolor{MacroColor}{#1}}%
  \else
    {\color{MacroColor}#1}%
  \fi
}
\newcommand{\defn}[1]{\textbf{#1}}
\newcommand{\defeq}{\mathrel{\stackrel{\textnormal{\tiny def}}{=}}}

\usepackage{xcolor}%
\newif\ifshowrevision          %
\showrevisiontrue              %

\ifshowrevision
  
\else
  
\fi

\newcommand{\wordcol}[7]{%
  \begin{scope}[xshift=#1]
    \ifnum#7=1
      \node[cell, fill=yellow!25, draw=yellow!60]
        at (0, -0.10) {};
    \fi
    \node[word] (w) at (0,0) {\strut #2};
    \node[stat, text=ctxcol!80!black] at (0,-0.62)
      {\textcolor{ctxcol}{#3}};
    \node[stat, text=unicol!80!black] at (0,-1.05)
      {\textcolor{unicol}{#4}};
    \node[stat, text=pmicol!80!black, font=\scriptsize\bfseries] at (0,-1.48)
      {\textcolor{pmicol}{#5}};
  \end{scope}
}

\definecolor{ctxcol} {HTML}{4B6BCA}
\definecolor{unicol} {HTML}{C0392B}
\definecolor{pmicol} {HTML}{27AE60}
\definecolor{rowA}   {HTML}{EEF2FB}
\definecolor{rowB}   {HTML}{FEF9EE}
\definecolor{hilight}{HTML}{FFF3B0}
\definecolor{hhilight}{HTML}{FFE5B4}
\definecolor{notecol}{HTML}{555555}

\newcommand{\awordinner}[6]{%
  \begingroup
  \setlength{\tabcolsep}{0pt}%
  \renewcommand{\arraystretch}{1.1}%
  \begin{tabular}[t]{c}
    \strut \makebox[3.0em][c]{{\footnotesize #1}} \\[1pt]
    \textcolor{ctxcol}{\scriptsize #2} \\
    \textcolor{unicol}{\scriptsize #3} \\
    \textcolor{unicol}{\scriptsize\itshape #4} \\
    \textcolor{pmicol}{\scriptsize\bfseries #5} \\
    \textcolor{pmicol}{\scriptsize\bfseries\itshape #6}
  \end{tabular}%
  \endgroup
}
\newcommand{\aword}[6]{\awordinner{#1}{#2}{#3}{#4}{#5}{#6}}
\newcommand{\awordh}[6]{%
  \awordinner{\colorbox{hilight}{\textbf{#1}}}{#2}{#3}{#4}{\textbf{#5}}{\textbf{#6}}%
}
\newcommand{\awordhh}[6]{%
  \awordinner{\colorbox{hhilight}{\textbf{#1}}}{#2}{#3}{#4}{\textbf{#5}}{\textbf{#6}}%
}

\newcolumntype{W}{>{\centering\arraybackslash}p{1.55cm}}
\newcolumntype{M}{>{\centering\arraybackslash}p{2.1cm}}  %

\usepackage{mathtools}
\usepackage{ellipsis}
\makeatletter
\renewcommand*{\mathellipsis}{%
  \mathinner{%
    \kern\ellipsisbeforegap%
    {\ldotp}\kern\ellipsisgap%
    {\ldotp}\kern\ellipsisgap%
    {\ldotp}\kern\ellipsisaftergap%
  }%
}
\renewcommand*{\dotsb@}{%
  \mathinner{%
    \kern\ellipsisbeforegap%
    {\cdotp}\kern\ellipsisgap%
    {\cdotp}\kern\ellipsisgap%
    {\cdotp}\kern\ellipsisaftergap%
  }%
}
\renewcommand*{\@cdots}{%
  \mathinner{%
    \kern\ellipsisbeforegap%
    {\cdotp}\kern\ellipsisgap%
    {\cdotp}\kern\ellipsisgap%
    {\cdotp}\kern\ellipsisaftergap%
  }%
}
\renewcommand*{\ellipsis@default}{%
  \ellipsis@before
  \kern\ellipsisbeforegap
  .\kern\ellipsisgap
  .\kern\ellipsisgap
  .\kern\ellipsisgap
  \ellipsis@after\relax}
\renewcommand*{\ellipsis@centered}{%
  \ellipsis@before
  \kern\ellipsisbeforegap
  .\kern\ellipsisgap
  .\kern\ellipsisgap
  .\kern\ellipsisaftergap
  \ellipsis@after\relax}
\AtBeginDocument{%
  \DeclareRobustCommand*{\dots}{%
    \ifmmode\@xp\mdots@\else\@xp\textellipsis\fi}}
\def\ellipsisgap{-.05em}
\def\ellipsisbeforegap{-.05em}
\def\ellipsisaftergap{.05em}
\makeatother

\usepackage[normalem]{ulem} %

\definecolor{ETHBlue}{RGB}{33,92,175}   %
\definecolor{ETHGreen}{RGB}{98,115,19}      %
\definecolor{ETHPurpleDark}{RGB}{140,10,89} %
\definecolor{ETHPurple}{RGB}{163,7,116} %
\definecolor{ETHGray}{RGB}{111,111,111} %
\definecolor{ETHRed}{RGB}{183,53,45}    %
\definecolor{ETHPetrol}{RGB}{0,120,148} %
\definecolor{ETHBronze}{RGB}{142,103,19}    %
\definecolor{ETHOrange}{RGB}{230, 100, 50}

\colorlet{MacroColor}{black}
\colorlet{MACROCOLOR}{MacroColor}

\definecolor{WildStrawberry}{rgb}{1.0, 0.26, 0.64}
\definecolor{Wisteria}{rgb}{0.79, 0.63, 0.86}
\definecolor{Olive}{rgb}{0.5, 0.5, 0.0}
\definecolor{Plum}{rgb}{0.56, 0.27, 0.52}
\definecolor{NavyBlue}{rgb}{0.0, 0.0, 0.5}
\definecolor{OliveGreen}{rgb}{0.33, 0.42, 0.18}
\definecolor{RoyalPurple}{rgb}{0.47, 0.32, 0.66}

\usepackage[capitalize,noabbrev]{cleveref}
\crefname{section}{\S}{\S\S}
\crefname{table}{Tab.}{Tab.}
\crefname{figure}{Fig.}{Figs.}
\crefname{algorithm}{Alg.}{}
\crefname{equation}{Eq.}{Eqs.}
\crefname{appendix}{App.}{Apps.}
\crefname{theorem}{Theorem}{}
\crefname{prop}{Proposition}{}
\crefname{definition}{Def.}{}
\crefname{cor}{Corollary}{}
\crefname{observation}{Observation}{}
\crefname{assumption}{Assumption}{}
\crefname{hyp}{Hyp.}{Hypotheses}
\crefformat{section}{\S#2#1#3}
\crefname{namedtheorem}{Hyp.}{Hypotheses}

\crefformat{equation}{\eqA Eq.~\eqB #2#1#3)}
\newcommand{\eqA}{}
\newcommand{\eqB}{(}

\newtheorem{proposition}{Proposition}

\usepackage{stmaryrd} %

\newcommand{\Var}{\mymacro{\mathbb{V}}}

\newcommand{\scon}{\mymacro{s_\text{context}}}

\newcommand{\suni}{\mymacro{s_\text{unigram}}}

\newcommand{\dll}{\mymacro{\Delta_\text{LL}}}

\newcommand{\vocab}{{\mymacro{\mathcal{W}}}}
\newcommand{\vocabitem}{{\mymacro{w}}}
\newcommand{\vocabitemSeq}{\mymacro{\textbf{\vocabitem}}}
\newcommand{\vocabmap}{\mymacro{\mathcal{R}}}
\newcommand{\vocabStrings}{\mymacro{\kleene{\vocab}}}

\newcommand{\alphabet}{{\mymacro{\Sigma}}}
\newcommand{\eoswordalphabet}{{\mymacro{\overline{\alphabet}}}}

\newcommand{\word}{{\mymacro{\sigma}}}
\newcommand{\wordseq}{{\mymacro{\boldsymbol{\sigma}}}}
\newcommand{\wordStrings}[0]{{\mymacro{\kleene{\alphabet}}}}
\newcommand{\wordstrnonempty}{\mymacro{\kleeneplus{\alphabet}}}

\newcommand{\kleene}[1]{#1^*}
\newcommand{\kleeneplus}[1]{#1^+}

\newcommand{\eos}{{\mymacro{\textsc{eos}}}}
\newcommand{\counter}{{\mymacro{\mathcal{\#}}}}

\DeclareTextCommandDefault{\textvisiblespace}{%
  \mbox{\kern.06em\vrule \@height.3ex}%
  \vbox{\hrule \@width.3em}%
  \hbox{\vrule \@height.3ex}}

\newcommand{\pcorpus}{p_{\mathrm{C}}} 
\newcommand{\phuman}{p_{\mathrm{H}}}

\definecolor{sfmcolor_v2}{RGB}{68, 114, 196}  

\definecolor{pcolor1}{RGB}{0,0,0}  %
\definecolor{pcolor2}{RGB}{0,0,0}     %
\definecolor{pcolor3}{RGB}{0,0,0}  

\newcommand{\ptag}[2]{%
  \setlength{\fboxsep}{1pt}%
  {\textbf{\color{#1!80!black}#2}}%
}

\newcommand{\pone}{\ptag{pcolor1}{Best estimate principle}}
\newcommand{\ptwo}{\ptag{pcolor2}{Cognitive plausibility principle}}
\newcommand{\pthree}{\ptag{pcolor3}{Consistent distribution principle}}

\newcommand{\X}[0]{{\mymacro{X}}}
\newcommand{\N}[0]{{\mymacro{N}}}
\newcommand{\subsampsize}[0]{{\mymacro{J}}}
\newcommand{\bootsampsize}[0]{{\mymacro{K}}}
\newcommand{\E}[0]{{\mymacro{\mathbbm{E}}}}

\newcommand{\LM}{{\mymacro{p_{\alphabet}}}}
\newcommand{\WordLM}{{\mymacro{p_{\vocab}}}}
\newcommand{\PrefixWordLM}{\mymacro{\prefixLanguageOp{\WordLM}}}
\newcommand{\prefixLanguageOp}[1]{\vec{#1}}

\newcommand{\prefixLM}{{\mymacro{\prefixLanguageOp{\LM}}}}

\newcommand{\freqNorm}{{\mymacro{u}}}
\newcommand{\mcest}{{\mymacro{\widehat{\freqNorm}}}}
\newcommand{\rbest}{{\mymacro{\widehat{\freqNorm^*}}}}

\renewcommand{\Pr}[2][]{\mathop{\mathbb{P}}_{\substack{#1}}\mleft[#2\mright]}

\newcommand{\canon}{\counter^{\mathrm{can}}}                 
\newcommand{\resid}{\counter^{\mathrm{res}}}                 
\newcommand{\mmccount}{\widehat{\counter}^{\,\mathrm{MMC}}}
\newcommand{\rbestapx}{\rbest_{\mathrm{can}}} 
\newcommand{\mmccountcan}{\widehat{\counter}^{\,\mathrm{MMC}}_{\mathrm{can}}}

\usepackage[safe]{tipa}
\newcommand{\uzurich}{\textnormal{1}}
\newcommand{\ethz}{\textnormal{2}}
\newcommand{\chifro}{\textnormal{3}}
\newcommand{\mpi}{\textnormal{4}}

\title{Revisiting the Estimation of Unigram Surprisal in Computational Psycholinguistics}
\author{Zachary Hopton\textsuperscript{\uzurich} \quad Francesco Ignazio Re\textsuperscript{\ethz}\quad Samuel Kiegeland\textsuperscript{\ethz,\,\chifro} \\ \textbf{Andreas Opedal}\textsuperscript{\ethz,\,\mpi} \quad
  \textbf{Eleanor Chodroff\textsuperscript{\uzurich} \quad Ryan Cotterell\textsuperscript{\ethz}} \\ \\
  \textsuperscript{\uzurich}University of Zurich~\;~\textsuperscript{\ethz}ETH Zürich~\;~\textsuperscript{\chifro}CHI-FRO~\;~%
  \textsuperscript{\mpi}MPI for Intelligent Systems, Tübingen \\
  \texttt{\href{mailto:zacharywilliam.hopton@uzh.ch}{zacharywilliam.hopton@uzh.ch} \href{mailto:ryan.cotterell@inf.ethz.ch}{ryan.cotterell@inf.ethz.ch}}
}

\begin{document}
\maketitle
\begin{abstract}
 Contextual and unigram surprisal are standard predictors of reading time in computational psycholinguistics. While contextual surprisal is typically derived from neural language models, unigram surprisal is usually obtained from independent corpus counts. Although both quantities are deterministic functionals of a single underlying distribution---capturing conditional and marginal structure respectively---this practice breaks their probabilistic link and may introduce estimation error from the mismatch. We propose three principles for estimating unigram surprisal in psycholinguistic research, including the requirement that both quantities be derived from the same language model. We operationalize this principle by introducing a Monte Carlo estimator of unigram surprisal from a language model, along with a marginalized variant that has lower variance empirically. When unigram surprisal is re-estimated under this unified framework, we find evidence that predictors change in their predictive contribution to reading time. We also confirm prior findings that frequency can be a stronger predictor of processing effort than contextual surprisal under this unified framework.

 \looseness=-1
\end{abstract}

\section{Introduction}
Human sentence processing is incremental
\citep{marslen1975sentence, tanenhaus1995integration, rayner2009language} and predictive \citep{altmann1999incremental, delong2005probabilistic}: comprehenders do not wait until a sentence is complete before interpreting it, but instead build partial representations and continuously generate expectations about upcoming input. Contextual surprisal---the negative log probability of a linguistic unit given its preceding context as derived from language models (LMs)---quantifies the degree to which incoming content aligns with these predictions. Surprisal has been found to robustly predict observed processing effort in humans across paradigms and languages \citep{haleprobabilistcearly, wilcox2020predictive, kuribayashi-etal-2022-context, shain2024large}, lending broad empirical support to expectation-based theories of comprehension \citep{levy2008expectation,kuperbergjaeger2016,ryskin2023prediction}. Separately, the frequency of a unit in a language also shapes how quickly it is processed: more frequent units are recognized faster, independent of their contextual fit \citep{hare1995frequency, bybee1998emergent, shain2024word}. This effect is captured by unigram surprisal---the negative log probability of a unit after marginalizing over all contexts. While estimating contextual surprisal has benefited greatly from advances in neural LMs \citep{goodkind-bicknell-2018-predictive, wilcox2020predictive, oh-schuler-2023-transformer, de-varda-marelli-2023-scaling, wilcox-etal-2023-language}, unigram surprisal has received comparatively little methodological attention. Typically, it is derived from small-data corpora or lexical databases \citep{robyn_speer_2022}.

\begin{figure}[t]
    \centering\small
    \colorbox{rowA}{%
      \parbox{\dimexpr\columnwidth-0.1\fboxsep\relax}{\vspace{3pt}%
        \begin{tabular}{@{}c@{\hspace{1.5pt}}c@{\hspace{1.5pt}}c@{\hspace{3pt}}c@{\hspace{3pt}}c@{\hspace{3pt}}|c@{\hspace{5pt}}c@{}}
          \aword{The}{21.75}{4.22}{16.04}{$-$17.53}{$-$5.72} &
          \aword{soup}{17.99}{16.11}{21.41}{$-$1.88}{3.42} &
          \aword{was}{6.77}{7.24}{13.54}{0.47}{6.77} &
          \aword{seasoned}{13.66}{18.24}{28.38}{4.57}{14.71} &
          \aword{with}{2.89}{7.14}{13.14}{4.25}{10.25} &
          \awordh{salt}{2.96}{14.62}{19.63}{11.66}{16.67} &
          \awordhh{coriander}{8.58}{20.10}{31.52}{11.52}{22.94}
        \end{tabular}\vspace{3pt}%
      }%
    }
    \caption{%
    \textcolor{ctxcol}{Contextual surprisal}, \textcolor{unicol}{unigram surprisal
    (\textsc{wf}\,/\,\emph{\textsc{MMC}})}, and \textcolor{pmicol}{PMI (\textsc{wf}\,/\,\emph{\textsc{MMC}})}
    for each word in the prefix \textit{The soup was seasoned with} and two target
    words \textit{\textbf{salt}} and \textit{\textbf{coriander}}. Unigram surprisal
    is obtained from two sources: the \texttt{wordfreq} package (\textsc{wf};  \citealp{robyn_speer_2022}) and a
    marginalized Monte Carlo estimate (\textsc{MMC}; italicized) from mGPT \citep{shliazhko-etal-2024-mgpt},
    which is also used for contextual surprisal.}
    \label{fig:example-sentence}
\end{figure}
\vspace{0em}

In a sense made specific by \citet{opedal-etal-2024-role}, contextual surprisal and unigram surprisal are correlated in naturalistic text. That being the case, modern studies on surprisal have controlled for unigram surprisal \citep{monsalve2012lexical, smith2013effect, shain-2019-large, wilcox-etal-2023-testing,  shain2024word}. However, controlling for unigram surprisal is only meaningful when both estimates come from the same distribution. In the example in \cref{fig:example-sentence}, we calculate \textcolor{ctxcol}{contextual surprisal} from the language model mGPT \citep{shliazhko-etal-2024-mgpt}, while we calculate \textcolor{unicol}{unigram surprisal} with two methods: estimates from the \texttt{wordfreq} package \citep{robyn_speer_2022} and estimates obtained with a marginalized Monte Carlo (MMC) estimator over the same language model as contextual surprisal. We observe that the \textcolor{pmicol}{pointwise mutual information (PMI)} (the difference between unigram and contextual surprisal) is nearly identical for \textit{salt} and \textit{coriander} under \texttt{wordfreq}, but diverges substantially under the MMC estimator. This illustrates the estimation error that can arise from a mismatch in unigram and contextual surprisal distributions. 

Motivated by this discrepancy, we put forward a set of three principles to guide approaches to estimating unigram surprisal for research on sentence processing and psycholinguistics more broadly: the best estimate principle, the cognitive plausibility principle, and the consistent distribution principle (see \cref{fig:intro-principles}; \cref{sec:principles}).
Though we argue that all three principles should guide unigram surprisal estimations, we note that no existing approach to estimating unigram surprisal satisfies the consistent distribution principle. Such an approach would eliminate the potential for error introduced by mismatched 
estimation procedures and allow practitioners to leverage the data scale of large LMs for unigram surprisal estimation. 
To that end, we introduce a Monte Carlo estimator that estimates unigram frequency from the same distribution as an LM by sampling strings from the model and calculating a unit's empirical frequency across these samples. We go on to derive a marginalized variant of the estimator, which we empirically demonstrate to be a more efficient estimator; that is, given the same number of samples, the marginalized variant will obtain lower variance estimates than the naive Monte Carlo estimator.
\looseness=-1

\begin{figure}[t]
    \centering
    \begin{tcolorbox}[colback=sfmcolor_v2!10, colframe=sfmcolor_v2!60,
                  title=\textbf{Principles}, arc=4pt, boxrule=0.8pt]
\fontsize{10}{13}\selectfont
\begin{description}[style=unboxed, leftmargin=0pt, labelindent=0pt]
  \item[\pone.] Unigram surprisal estimates should closely approximate the marginal probability of units according to the true distribution over units that generated the corpus used for estimation.
  \item[\ptwo.] Estimates of unigram surprisal should aim to approximate the marginal probability of units according to the human LM.
  \item[\pthree.] Unigram and contextual surprisal estimates should be derived from the same underlying LM.\looseness=-1
\end{description}
\end{tcolorbox}
    \vspace{-10pt}
    \caption{Principles for estimating unigram surprisal.}
    \label{fig:intro-principles}
\end{figure}
\looseness=-1

In the empirical portion of the article,
we study several predictors of reading times across eleven languages, as in \citet{opedal-etal-2024-role}, and directly compare models of reading time fit using unigram surprisal values derived from the \texttt{wordfreq} lexical database \citep{robyn_speer_2022} and values from our MMC estimator.
Like \citet{opedal-etal-2024-role}, we find that unigram surprisal is a stronger predictor of reading time than contextual surprisal; however, we find that the predictive power of some predictors significantly changes when unigram surprisal is calculated with our MMC estimator. 
Though our empirical findings focus specifically on sentence processing, they indicate that the estimation of unigram surprisal should be reconsidered more broadly in computational psycholinguistics.

\section{Language Models}
\label{sec:prelim}

\paragraph{Preliminaries.}

An alphabet is a finite, non-empty set. Let $\alphabet$ be an alphabet of tokens.
We denote elements $\word \in \alphabet$ by lowercase Greek letters, and strings
$\wordseq = \word_1 \cdots \word_T$ by bold lowercase Greek letters, where
$\word_t$ denotes the $t^{\text{th}}$ token of $\wordseq$.
Furthermore, let $\vocab$ be an inventory of words, i.e., a non-empty set which may be finite or countably infinite \citep{snbjarnarson2026transducing, kiegeland-etal-2026-proper}. We denote elements of $\vocab$ by $\vocabitem$.
$\wordStrings$ denotes the set of all finite strings formed by concatenating tokens in $\alphabet$, including the empty string $\varepsilon$.
We write $\lvert \wordseq \rvert = T$ for the length of $\wordseq = \word_1 \cdots \word_T$.
A \defn{prefix} of $\wordseq$ is an initial substring; specifically, $\wordseq_{<t}$ consists of its first $t-1$ tokens.
We write $\wordseq' \preceq \wordseq$ to indicate that $\wordseq'$ is a prefix of $\wordseq$, and $\wordseq' \prec \wordseq$ for a \emph{strict} prefix.
Concatenation of $\wordseq$ and $\wordseq'$ is written $\wordseq\wordseq'$.
Finally, we introduce a distinguished end-of-string symbol $\eos \notin \alphabet$ and define the augmented alphabet $\eoswordalphabet \defeq \alphabet \cup \{\eos\}$.

\paragraph{Language models.}
Given an alphabet $\alphabet$, a \defn{language model} (\defn{LM}) is a probability distribution over $\wordStrings$.
We denote a language model as $\LM$
and use $\X$ to denote a random variable taking values in $\wordStrings$ distributed according to $\LM$.
Given a string $\wordseq \in \wordStrings$, its prefix probability is:\footnote{We remark that the prefix probability is \underline{not} a probability distribution in general, since it is unnormalized.}
\begin{subequations}
\begin{align}
\prefixLM\left(\wordseq\right) &\defeq \underset{\X\sim \LM}{\mathbb{P}}\left[\wordseq \preceq \X\right] \\ 
&= \sum_{\wordseq'\in\wordStrings} \mathbbm 1 \left\{\wordseq \preceq \wordseq'\right\} \LM\left(\wordseq'\right).
\end{align}
\end{subequations}
That is, the prefix probability of $\wordseq$ is the sum over the probabilities of all possible continuations of $\wordseq$. 

\paragraph{Contextual probability.}
The probability of a string can be factored autoregressively using prefix-conditioned probabilities. We assume that $\forall \wordseq \in \wordStrings \colon \LM\!\left(\wordseq\right) > 0$.\footnote{This holds true for transformer-based LMs due to the softmax function.} 
First, we define the conditional prefix probability
of a token $\word \in \alphabet$:
\begin{equation}
    \prefixLM\left(\word\mid\wordseq\right) \defeq \frac{\prefixLM\left(\wordseq\word\right)}{\prefixLM\left(\wordseq\right)}.
\end{equation}
The prefix probability of $\eos$ is defined as
\begin{equation}
    \prefixLM\left(\eos\mid \wordseq\right) = \frac{\LM\left(\wordseq\right)}{\prefixLM\left(\wordseq\right)}.
\end{equation}
Thus, we will also use $\prefixLM\left(\cdot \mid \wordseq\right)$ to denote a probability distribution over $\eoswordalphabet$.
Using the conditional prefix probabilities, there always exists a left-to-right autoregressive factorization of the probability of a string under an LM:
\begin{equation}
    \LM\left(\wordseq\right) = \prefixLM\left(\eos\mid\wordseq\right)\prod^{T}_{t=1}\prefixLM\left(\word_t\mid\wordseq_{<t}\right).
\end{equation}

\paragraph{Unigram probability.}
We define the unigram 
probability of a word $\vocabitem \in \vocab$ in terms of the frequency with which it occurs in strings sampled from a language model $\LM$. Following \citet{kiegeland-etal-2026-proper}, we assume a deterministic \defn{unit parser} $\pi \colon \wordStrings \rightarrow \vocabStrings$ mapping token strings to word strings.\footnote{Parsing into units is ambiguous in general (see \citet{kiegeland-etal-2026-proper}). Here, we adopt the assumption that parsing is deterministic, which holds for the units we study.} Its inverse, the \defn{realization}, is a relation, since a word string is realized by many token strings. We further assume the positions of $\pi(\wordseq)$ are realized by pairwise disjoint, nonempty, contiguous token spans of $\wordseq$, in left-to-right order, so that $\lvert\pi(\wordseq)\rvert \le \lvert\wordseq\rvert$. The pushforward of $\LM$ through $\pi$ is then defined as
\begin{equation}
\label{eq:pushforward}
\WordLM(\boldsymbol{\vocabitem}) \defeq \sum_{\wordseq\in\pi^{-1}(\boldsymbol{\vocabitem})}
\LM(\wordseq), \qquad \boldsymbol{\vocabitem}\in\vocabStrings.
\end{equation}
We write $\PrefixWordLM$ for the prefix probability of $\WordLM$, so that $\PrefixWordLM\left(\vocabitem \mid \vocabitemSeq\right)$ is the conditional next-word probability. To account for the fact that a word may admit multiple tokenizations, we write
\[
\vocabmap(\vocabitem) \defeq \left\{\wordseq \in \wordstrnonempty \colon \pi(\wordseq) = \vocabitem\right\}
\]
for the set of token strings realizing $\vocabitem$ alone (identifying $\vocabitem$ with the length-one word string), where $\wordstrnonempty = \wordStrings \setminus \{\varepsilon\}$.
For a token string $\wordseq$, the number of occurrences of $\vocabitem$ is read off its parse:\footnote{A target spanning several units is treated as the corresponding word subsequence, with the count defined analogously.}
\begin{equation}
    \counter\!\left(\wordseq,\vocabitem\right)
\defeq
\sum_{t=1}^{\lvert\pi\left(\wordseq\right)\rvert}
\mathbbm{1}\!\left\{\pi\left(\wordseq\right)_t=\vocabitem\right\}.
\end{equation}
\Cref{ap:boundaries} describes how we implement $\pi$, and hence $\counter$, for the tokenizer used in our experiments.
The expected count of a word $\vocabitem \in \vocab$ under $\LM$ is then
\begin{equation}
\label{eq:freq01}
\begin{aligned}
    &\mu\!\left(\vocabitem\right) \defeq \E_{\X\sim \LM}\!\left[\counter\!\left(\X, \vocabitem\right)\right]= \\ 
    &\sum_{\wordseq\in\wordStrings} \LM\!\left(\wordseq\right)\counter\!\left(\wordseq,\vocabitem\right).
\end{aligned}
\end{equation}
We can then derive a notion of a word's unigram probability by normalizing this expectation to a proper probability distribution over $\vocab$:
\begin{equation}
\label{eq:freq02}
    \freqNorm\!\left(\vocabitem\right) = \frac{\mu\!\left(\vocabitem\right)}{\sum_{\vocabitem' \in \vocab}\mu\!\left(\vocabitem'\right)}.
\end{equation}
Note that---like the contextual prefix probability $\prefixLM\!\left(\cdot \mid \wordseq\right)$---$\freqNorm\!\left(\vocabitem\right)$ is a distribution over $\vocab$, and is only defined for a single word $\vocabitem \in \vocab$.

\noindent\textit{Remark.} Let $N_{\vocab}(\wordseq)$ denote the number of occurrences in $\wordseq$ of words in $\vocab$. By linearity of expectation,
$\sum_{\vocabitem\in\vocab}\mu\!\left(\vocabitem\right)
=\E_{\X\sim\LM}\!\big[\textstyle\sum_{\vocabitem\in\vocab}\counter\!\left(\X,\vocabitem\right)\big]
=\E_{\X\sim\LM}\!\left[N_{\vocab}(\X)\right]$, where the last equality is immediate, since each position of $\pi(\wordseq)$ is occupied by exactly one word of $\vocab$. Thus \cref{eq:freq02} normalizes by the expected word count.

\section{Which Language Model?}\label{sec:which-lm}

Formal theories of language, most notably the generative tradition
\citep{chomsky1957syntactic,chomsky1965aspects}, have historically separated a
speaker's idealized knowledge of grammatical rules (\defn{competence}) from the
actual, often imperfect, use of language in real-time processing (\defn{performance}).
Probability has found a central role in explaining the latter
\citep{bod-etal-2003-probabilistic,jurafsky2003probabilistic,chater2006probabilistic,levy2008expectation}: under a probabilistic, performance-oriented view
a model of language is then evaluated not by whether it generates the correct set of
grammatical strings, but by how well it mirrors the behavior of human speakers. The
question is thus not \emph{whether} probability is relevant for modeling linguistic
performance, but \emph{what} distribution it should approximate. We contrast two
viewpoints on this question, which we refer to as the frequentist and the cognitive
views.\looseness=-1

\paragraph{The frequentist view.}
One approach to estimating an LM
is to optimize it to place high
probability on a large, heterogeneous corpus---comprising, e.g., collaborative
text on Wikipedia, conversational text on Reddit, and the structured,
iteratively revised prose of digitized books
\citep{radford2019language,groeneveld-etal-2024-olmo}.\footnote{Instruction-tuned \citep{ouyang2022instruction} and reasoning-oriented \citep{deepseekai2025deepseekr1incentivizingreasoningcapability} variants involve
additional training objectives that shift the distribution away from raw
corpus frequency, complicating a straightforwardly frequentist interpretation
of the resulting probabilities.}
Under the frequentist view, this LM approximates the probability
distribution that generated the corpus text, which we denote $\pcorpus$. 
The cognitive relevance of $\pcorpus$ has been widely studied---corpus frequency has been shown to play a central role in language
use and change \citep{bybee1998emergent,bybee_introduction_2001}, accounting for
phonological reduction 
\citep{bybee1999effect,bybee2002word} and
grammaticalization \citep{haspelmath_why_1999}. On the other hand, the
generative process underlying such heterogeneous corpora is complex: it
involves millions of individuals writing diverse texts in an iterative fashion
over extended periods of time, and the resulting corpora are several orders of
magnitude larger than the linguistic input any single human observes over a
lifetime \citep{warstadt2022what}.
\looseness=-1

\paragraph{The cognitive view.}
In contrast, the cognitive view interprets the LM as a proxy for
$\phuman$---a theoretical construct representing the internalized probability
distribution over linguistic strings that a human
carries in their mind. 
Because we lack direct access to such a theoretical construct, a corpus-estimated
LM often serves as a practical approximation, an assumption
underlying much of psycholinguistic modeling
\citep[\emph{inter alia}]{wilcox2020predictive,oh-schuler-2023-surprisal,shain2024large}.
However, as LMs have grown larger---in terms of both training data
and parameter count---they have become \emph{less} predictive of human
behavior. This phenomenon is known as the inverse-scaling law
\citep{oh-schuler-2023-surprisal}: predictive accuracy of human data improves
with scale only up to a point, after which it degrades.
This motivates an alternative approach in which the LM is directly aligned with human
psychometric data to obtain more cognitively plausible probability estimates
\citep{kiegeland-etal-2024-reverse}.
The construct $\phuman$ need not be an unconstrained distribution over
arbitrarily long contexts, either. The cognitive model may be \emph{lossy}
\citep{futrell2020lossy}, reflecting constraints such as limited working memory
\citep{gibson1998linguistic,Gibson2000TheDL} that restrict how much prior
context is effectively encoded; in the limit, $\phuman$ reduces to unigram
probabilities.\looseness=-1
\section{Surprisal Theory}
\label{sec:surp}

\paragraph{Eye tracking and reading data.}

To study sentence processing during reading, several corpora have been developed in which many participants read the same texts and eye-tracking data are collected while reading \citep{kennedy2003dundee, siegelman2022expanding}. 
Eye-tracking studies record participants' gaze continuously, allowing fixations of varying duration to be mapped directly onto the text \citep{rayner1998eye}. Although some work has sought to model these fine-grained fixations \citep{re-etal-2025-spatio}, the fixations are generally aggregated to yield a measure of reading time for some unit of interest. The choice of aggregation reflects assumptions about which stage of processing is being targeted \citep{rayner1998eye,frank2013reading, berzak2023eye}: \emph{first fixation duration} (the duration of the reader's first fixation on a unit) is thought to reflect early lexical processing \citep{clifton-2007-eye}, while \emph{gaze duration} (cumulative fixation time before first leaving a unit) and \emph{total fixation duration} (cumulative fixation time including regressions back to the unit) capture progressively later integration processes \citep{liversedge-etal-1998eye}. \looseness=-1

\paragraph{Contextual surprisal.}

Surprisal theory states that processing difficulty is related to a unit's negative log probability in a given context \citep{haleprobabilistcearly, levy2008expectation, shain2024large}. Psycholinguistic studies test this theory by relating surprisal to
behavioral or neural measures of processing effort such as fixation duration from eye-tracking data \citep{wilcox-etal-2023-testing} and event-related potentials (ERPs) from electroencephalography \citep{frank-etal-2013-word}. LMs play a central role in these studies because they provide estimates of contextual probabilities. Given $\WordLM$, the contextual surprisal of a unit $\vocabitem_t$ in context $\vocabitemSeq_{<t}$ is defined as
\begin{equation}
    \label{eq:con_surp}
    \scon\!\left(\vocabitem_t, \vocabitemSeq_{<t}\right) = -\log \PrefixWordLM\!\left(\vocabitem_t \mid \vocabitemSeq_{<t}\right).
\end{equation}

\paragraph{Unigram surprisal.}
Psycholinguistic studies indicate that units with high unigram frequency are more easily accessible in the mental lexicon \citep{hare1995frequency, bybee1998emergent}.
This is typically measured in log scale, using \defn{unigram surprisal} 
\begin{equation}
    \label{eq:uni_surp}
    \suni\!\left(\vocabitem\right) = -\log \freqNorm\!\left(\vocabitem\right),
\end{equation}
where $\vocabitem \in \vocab$ is the unit of interest. Both surprisal quantities are thus functionals of the single distribution $\LM$ (given the unit parser): contextual surprisal of the conditional structure it induces through $\PrefixWordLM$, and unigram surprisal of the marginal it induces through $\freqNorm$. The consistent distribution principle (\cref{sec:principles}) demands that both be estimated from the same $\LM$, so that this relationship also holds for the estimates.
Given that it has been shown to affect ease of processing,
unigram surprisal is regularly included as a baseline predictor of reading times; however, some studies have focused specifically on its effect on reading \citep{shain-2019-large,shain2024word,opedal-etal-2024-role}.

\section{Principles of Frequency Estimation}
\label{sec:principles}

Having discussed which distributions can be modeled to study language and having defined unigram and contextual surprisal, we can now formalize three guiding principles for estimating unigram surprisal and discuss relevant literature around them. The principles are summarized in \cref{fig:intro-principles}.

\paragraph{\pone.}
The first principle adopts an empirical conception of unigram frequency. Under this conception, unigram frequency is treated as a real population quantity: the relative prevalence of a unit in the underlying distribution that generated the corpus. The observed strings are then viewed as samples from that distribution, and estimation amounts to recovering this empirical truth from data. This perspective has two advantages. First, it ties the target quantity to observable language use rather than to an unobserved latent construct. Second, it admits a simple and interpretable estimator: under independent and identically distributed (i.i.d.) samples,
unigram probabilities are estimated by normalized token counts, which is equivalent to the maximum-likelihood estimator for a multinomial model; see \cref{app:mle_est}. As we show in \cref{prop:plugin_bias}, because $-\log$ is strictly convex, an unbiased estimator of unigram probability does not yield an unbiased estimator of unigram surprisal. Plug-in estimators of related information-theoretic quantities (i.e., entropy) are known to be biased for the same reason \citep{miller1955,arora-etal-2022-estimating}. Even for an unbiased unigram probability estimator, the estimated unigram surprisal is biased upward. The bias in surprisal disappears only for a zero-variance unigram probability estimator, motivating lower variance frequency estimators. \looseness=-1

\paragraph{\ptwo.}
The second principle adopts a cognitive conception of unigram frequency. Under this conception, the target is not an empirical population quantity directly given by attested token counts, but a latent property of the expectations in the human mind, as discussed with respect to the cognitive view of the LM (\cref{sec:which-lm}). 
The goal is not to simply recover the unigram distribution of the corpus-generating process, but to approximate the distribution that best reflects human baseline expectations over units.
In the past, this has been operationalized as the extent to which frequency values are predictive of human reaction times in lexical decision tasks. This criterion has been studied as an alternative way to validate frequency resources, rather than factors like corpus size \citep{keuleers2010subtlex, brysbaert2011word}. This has led to findings that texts from the Internet \citep{balota2004visual} or movie and TV subtitles \citep{new2007use} seem to yield more cognitively plausible frequency values than those from edited publications. Importantly, this principle may be at odds with the best estimate principle. Larger LMs, despite achieving lower perplexity, have been found to yield contextual surprisal estimates that are less predictive of processing effort than those of smaller models \citep{oh_comparison_2022, oh-schuler-2023-surprisal}.
One explanation for this is that larger models assign probabilities to very low-frequency units more accurately, but in ways that are less cognitively plausible \citep{oh-etal-2024-frequency}. Similarly, \citet{opedal-etal-2024-role} argue that medium-sized models may yield contextual surprisal estimates closer to unigram frequency because they are less effective at predicting rare units in context, better matching the balance of contextual and frequency-based information in human processing.
The same issue arises for unigram frequency itself: the statistically best estimate need not be the cognitively most plausible one. Although this question has been studied extensively for contextual surprisal, comparatively little is known about how modeling choices affect the cognitive plausibility of unigram frequency estimates.
\looseness=-1

\paragraph{\pthree.}
The third principle is that contextual and unigram surprisal should be derived from the same distribution. Although the literature often
estimates them from different sources, 
the two quantities are not independent: both are functionals of a single distribution over linguistic sequences. Contextual surprisal reflects its conditional structure, while unigram surprisal reflects the marginal distribution it induces over units. Under a shared distribution, this relationship holds by construction. Under heterogeneous estimation, it does not: the two predictors then summarize different probabilistic models rather than a single generative system. Our unified estimation framework restores this coherence by deriving both quantities from the same LM. The consistent distribution principle has been observed before, though it has primarily been addressed by changing the corpora used to calculate unigram frequency. For example, \citet{shain2024word} calculate contextual surprisal with GPT-2. To calculate unigram surprisal,  \citeauthor{shain2024word} use a corpus that was specifically chosen to approximate the model's training distribution in an effort to minimize effects from this mismatch. Likewise, in their assessment of surprisal as a measure of metaphor novelty, \citet{momen-zarriess-2026-frequency} compare unigram surprisals from \texttt{wordfreq} \citep{robyn_speer_2022} to those calculated over the training data of the LMs they use to obtain contextual surprisal.
\looseness=-1

\section{From Autoregressive Language Models to Unigram Probabilities}
\label{sec:estimation}

We now consider how to operationalize the consistent distribution principle, taking an autoregressive language model as the shared distribution and studying the problem of estimating the unigram probabilities it induces. Since exact computation of \cref{eq:freq02} is infeasible, we instead consider two different estimators of this quantity. Specifically, we define estimators for words $\vocabitem \in \vocab$, which may have several valid tokenizations under a language model's token alphabet $\Sigma$.\footnote{Even with a deterministic encoder, LMs can assign mass to non-canonical tokenizations \citep{vieira2025from}.}

\subsection{A Naive Monte Carlo Estimator}

A straightforward way to estimate \cref{eq:freq02} is by naive Monte Carlo (NMC) sampling. 
Here, we independently sample $\N$ strings from $\LM$, i.e., $\X^{\!\left(1\right)}, \ldots, \X^{\!\left(\N\right)} \stackrel{\text{i.i.d.}}{\sim} \LM$, and estimate $\freqNorm\!\left(\vocabitem\right)$ by the empirical frequency of $\vocabitem$ across all word occurrences in the samples:
\begin{equation}
    \label{eq:mc_est}
    \mcest\!\left(\vocabitem\right) =
    \frac{\sum_{n=1}^\N \counter\!\left(\X^{\!\left(n\right)},\vocabitem\right)}
         {\sum_{n=1}^\N \sum_{\vocabitem' \in \vocab} \counter\!\left(\X^{\!\left(n\right)},\vocabitem'\right)}.
\end{equation}

\begin{proposition}
\label{prop:mc_consistency}
Assume that $\E_{\X \sim \LM}\!\left[\lvert \X \rvert\right] < \infty$ and that $\sum_{\vocabitem' \in \vocab}\mu\!\left(\vocabitem'\right) > 0$.
Then, for every fixed \(\vocabitem \in \vocab\), the Monte Carlo estimator
\(\mcest_{\N}\!\left(\vocabitem\right)\) is strongly consistent and asymptotically unbiased for
\(\freqNorm\!\left(\vocabitem\right)\). That is,
\begin{equation}
    \mcest_{\N}\!\left(\vocabitem\right)
    \xrightarrow{a.s.}
    \freqNorm\!\left(\vocabitem\right)
    \qquad \text{as } \N \to \infty,
\end{equation}
and
\begin{equation}
    \lim_{\N \to \infty}
    \E\!\left[\mcest_{\N}(\vocabitem)\right]
    =
    \freqNorm\!\left(\vocabitem\right).
\end{equation}
\end{proposition}

\subsection{A Marginalized Monte Carlo Estimator}
\label{sec:rb}
The naive MC estimator in \cref{eq:mc_est} can have high variance, since $\X$ ranges over a very large (infinite) sample space.\footnote{See, e.g., \citet{giulianelli-etal-2024-generalized}, who report high variance when using Monte Carlo sampling to estimate sequence-level entropy in psycholinguistic modeling. \citet{clark-etal-2025-well} instead advocate for estimating entropy at the word level.} This has negative implications for the efficiency of the frequency estimator, as well as the unigram surprisal values which we derive from it. As we show in \cref{prop:plugin_bias}, unigram surprisal values are positively biased in expectation unless the estimator has zero variance, further motivating a lower-variance frequency estimator.\looseness=-1

To reduce this variance, we propose a \defn{marginalized Monte Carlo} (MMC) estimator. The key idea is to retain the sampled prefixes of the NMC estimator, while marginalizing over the next token at each prefix. Rather than treating a word as present or absent in the sampled next token, we compute the conditional expected contribution of that next token to the word count. Suppose a sampled sequence has reached a prefix $\wordseq$. NMC observes a single sampled continuation and therefore records either an occurrence of $\vocabitem$ at that position or no occurrence. The LM, however, assigns a probability to \emph{every} next token, as well as to the further tokens needed to complete an occurrence of $\vocabitem$; MMC uses all of this information.\looseness=-1

\paragraph{Canonical and residual occurrences.}
An occurrence of $\vocabitem$ may be realized by different token strings: its realizing span (\cref{ap:mmc_noncanon}), read as a standalone string, is an element of $\vocabmap\!\left(\vocabitem\right)$ under the parser of \cref{ap:boundaries}. We split the count into the part realized by the canonical tokenization of $\vocabitem$, and the part realized in any other way:
\begin{equation}
    \label{eq:can_res}
    \counter\left(\wordseq,\vocabitem\right) = \canon\left(\wordseq,\vocabitem\right) + \resid\left(\wordseq,\vocabitem\right).
\end{equation}
The canonical part can be marginalized with the LM's next-token distributions, whereas the residual part is only read off a sampled string once that string is complete.\looseness=-1

\paragraph{Marginalizing at each prefix.}
We use $Z^{\mathrm{can}}_{\vocabitem,t}$ to refer to a canonical occurrence of $\vocabitem$ beginning at position $t$ (formally, \cref{ap:mmc_noncanon}), and define
\begin{equation}
    \label{eq:h_def}
    h_{\vocabitem}\left(\word\mid\wordseq\right)
    \defeq
    \E\left[Z^{\mathrm{can}}_{\vocabitem,t}\mid \X_{<t}=\wordseq,\, \X_t=\word\right],
\end{equation}
the probability that a canonical occurrence begun at position $t$ is completed---including its right-boundary condition (\cref{ap:boundaries})---given that the next token is $\word$. Averaging over the next-token distribution, $\sum_{\word\in\alphabet}\prefixLM\!\left(\word\mid\wordseq\right)h_{\vocabitem}\left(\word\mid\wordseq\right)$ is the expected canonical contribution of the position following $\wordseq$.\looseness=-1

Marginalizing the canonical part alone would discard residual occurrences. Writing $H_{\vocabitem}\left(\wordseq\right)\defeq\sum_{s=1}^{\lvert\wordseq\rvert}h_{\vocabitem}\!\left(\word_s\mid\wordseq_{<s}\right)$ for the marginalized canonical contribution accumulated along $\wordseq$, we define a per-step contribution that charges whatever is left over to the terminal $\eos$ position:
\begin{equation}
    \label{eq:d_contrib}
    d_{\vocabitem}\left(\word\mid\wordseq\right)
    \defeq
    \begin{cases}
    h_{\vocabitem}\left(\word\mid\wordseq\right), & \word\in\alphabet,\\[4pt]
    \counter\left(\wordseq,\vocabitem\right)-H_{\vocabitem}\left(\wordseq\right), & \word=\eos.
    \end{cases}
\end{equation}
By construction, this decomposition is exact pathwise: for every completed string $\wordseq$,
\begin{equation}
    \label{eq:pathwise}
\counter\!\left(\wordseq,\vocabitem\right)\!=\!\sum_{t=1}^{\lvert\wordseq\rvert}\!d_{\vocabitem}\left(\word_t\!\mid\!\wordseq_{<t}\right)\!+\!d_{\vocabitem}\left(\eos\!\mid\!\wordseq\right).
\end{equation}
The $\eos$ term therefore absorbs both the residual, non-canonical occurrences and the gap between the marginalized and the realized canonical counts.\looseness=-1

\paragraph{The MMC estimator.}
Taking the conditional expectation of $d_{\vocabitem}$ at every prefix visited by a sampled string $\X$ gives the per-string count estimate
\begin{equation}
\label{eq:mmc_noncanon}
\begin{aligned}
&\mmccount\left(\X,\vocabitem\right)
\defeq \\
&\quad\sum_{t=1}^{\lvert\X\rvert+1}
\E_{A\sim \prefixLM\left(\cdot\mid \X_{<t}\right)}
\left[
d_{\vocabitem}\left(A\mid \X_{<t}\right)
\right],
\end{aligned}
\end{equation}
which we show to be unbiased for the expected count of \cref{eq:freq01} in \cref{ap:mmc_noncanon}. Normalizing exactly as in \cref{eq:mc_est} yields our estimator of \cref{eq:freq02}:
\begin{equation}
    \label{eq:rb_est_02}
    \rbest\left(\vocabitem\right) =
    \frac{\sum_{n=1}^\N \mmccount\left(\X^{\left(n\right)},\vocabitem\right)}
         {\sum_{n=1}^\N \sum_{\vocabitem' \in \vocab} \counter\left(\X^{\left(n\right)},\vocabitem'\right)},
\end{equation}
where $\X^{\!\left(1\right)}, \ldots, \X^{\!\left(\N\right)} \stackrel{\text{i.i.d.}}{\sim} \LM$. 
The denominator is the same expected-word-count estimate used by NMC in \cref{eq:mc_est}.\looseness=-1

\paragraph{Approximate variant.}
Given the computational cost of the MMC estimator (\cref{ap:mmc_noncanon}) and the relatively small probability mass that LMs assign to non-canonical tokenizations \citep{vieira2025from}, we also experiment with an approximation that counts only canonical occurrences, i.e., replaces $\counter$ in \cref{eq:can_res} by $\canon$, so that $\resid\left(\wordseq,\vocabitem\right) = 0$. The $\eos$ term of the expectation-equivalent contribution $\widetilde d_{\vocabitem}$ of \cref{ap:mmc_noncanon} then vanishes, and the corresponding estimator reduces to the canonical-only count estimate $\mmccountcan\left(\X,\vocabitem\right) \defeq \sum_{t=1}^{\lvert\X\rvert+1}\E_{A\sim \prefixLM\left(\cdot\mid \X_{<t}\right)}\left[h_{\vocabitem}\!\left(A\mid \X_{<t}\right)\right]$ (with $h_{\vocabitem}\left(\eos \mid \cdot\right) \defeq 0$), which is unbiased for the canonical share $\E_{\X \sim \LM}\!\left[\canon\left(\X, \vocabitem\right)\right]$ of \cref{eq:freq01}, and so agrees with \cref{eq:mmc_noncanon} in expectation exactly when the residual mass vanishes. Normalizing as before gives
\begin{equation}
    \label{eq:mmc_approx}
    \rbestapx\left(\vocabitem\right)\!=\!
    \frac{\sum_{n=1}^\N \mmccountcan\!\left(\X^{\left(n\right)},\vocabitem\right)}
         {\sum_{n=1}^\N \sum_{\vocabitem' \in \vocab} \counter\!\left(\X^{\left(n\right)},\vocabitem'\right)}.
\end{equation}
The approximation above does not require taking non-canonical tokenizations in the sampled strings into account, meaning that its numerator underestimates the expected count $\mu\left(\vocabitem\right)$ whenever residual mass is nonzero. %
We study the downstream impact of this in \cref{sec:exp}.

\section{Experiments}
\label{sec:exp}

\citet{opedal-etal-2024-role} model reading times of orthographic words in eleven 
languages using contextual surprisal from mGPT \citep{shliazhko-etal-2024-mgpt} and 
unigram surprisal from the \texttt{wordfreq} package \citep{robyn_speer_2022}, with 
data from the Multilingual Eye-movement Corpus (MECO; \citeauthor{siegelman2022expanding}, 
\citeyear{siegelman2022expanding}). 
We extend their analyses, retaining their 
contextual surprisal values but replacing unigram surprisal estimates obtained from \texttt{wordfreq} with those estimated from the same distribution as the contextual surprisal using our NMC and MMC estimators. 
Strings for estimation are obtained by sampling from mGPT with a maximum length of $50$ for computational considerations, and estimates are converted to unigram surprisals by taking the negative log of the estimated unigram probabilities. 

\citet{opedal-etal-2024-role} compare the predictive power of unigram surprisal, contextual 
surprisal, and pointwise mutual information (PMI). PMI \citep{fano1961transmission} measures 
the association between a word $\vocabitem_t$ and its context $\vocabitemSeq_{<t}$, and is defined as:
\looseness=-1
\begin{equation}
    \label{eq:pmi}
    \text{PMI}\!\left(\vocabitem_t, \vocabitemSeq_{<t}\right) = \suni\!\left(\vocabitem_t\right)-\scon\!\left(\vocabitem_t, \vocabitemSeq_{<t}\right).
\end{equation}\looseness=-1

\noindent Wherever 
PMI or unigram surprisal appear in our models, we report results using both the original 
\texttt{wordfreq} estimates \citep{robyn_speer_2022} and each of our approaches to  estimating unigram surprisal from mGPT. We model reading times using generalized additive mixed models (GAMMs), which differ from linear models by allowing predictors to enter through smooth, nonparametric functions rather than fixed linear coefficients. Our dependent variable is gaze duration. We fit five GAMMs with smooth terms for
(i) unigram surprisal, (ii) contextual surprisal, (iii) PMI,
(iv) contextual surprisal and unigram surprisal, and
(v) PMI and unigram surprisal, each with a single spillover term per predictor and random effect smooths for participant and trial. 
Model fit is assessed via delta log-likelihood ($\dll$), 
defined as the difference in per-unit log-likelihood between a candidate model 
and a baseline with only random effects.
Predictions are obtained via 10-fold cross-validation;
positive values indicate that the candidate model assigns higher likelihood to the observed reading times than this baseline. To assess whether the MMC estimators 
yield a different quality of fit than corpus-based unigram surprisals, we compare 
$\dll$ values between the two frequency conditions at the unit level 
using paired permutation tests, with $p$-values 
corrected for multiple comparisons using the Holm--Bonferroni procedure. We discuss MECO and the covered languages, preprocessing, model specification, and evaluation details in  \cref{sec:ap_data}.

\begin{figure*}[t]
    \centering 
    \includegraphics[width=\linewidth]{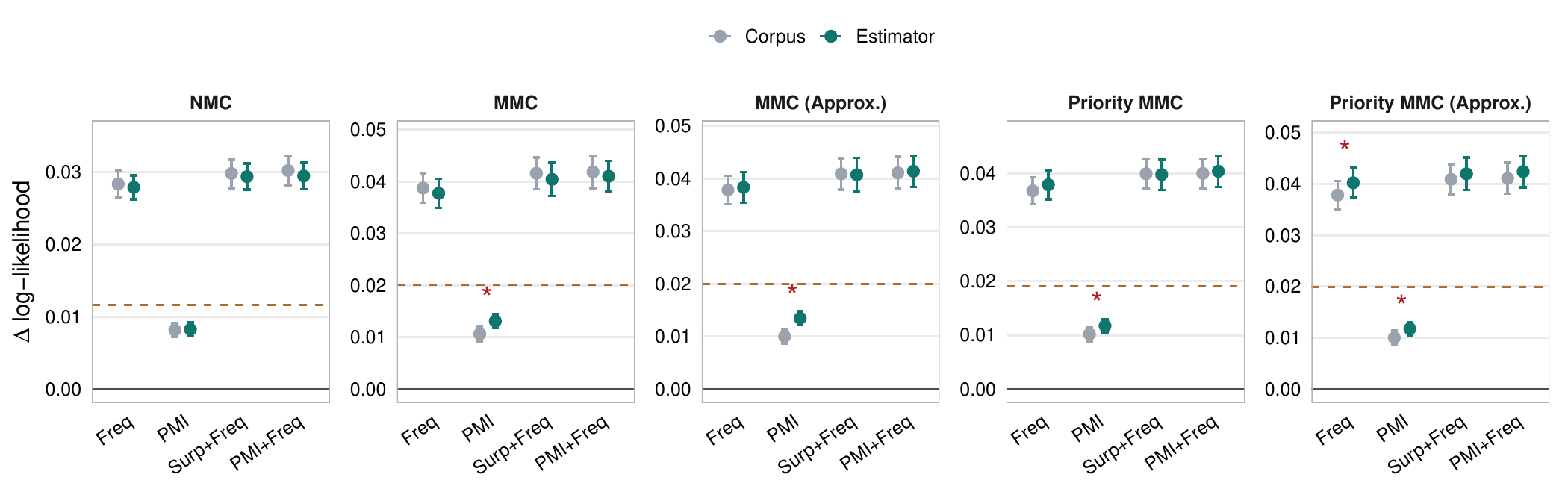}
    \caption{Comparison of $\dll$ for various GAMMs of orthographic word gaze duration fit with unigram surprisals obtained with the \texttt{wordfreq} package \citep{robyn_speer_2022} or with one of our estimators using samples from mGPT. Values are averaged across languages. Dashed lines represent the $\dll$ of adding contextual surprisal relative to each facet's baseline for reference. Significance tests compare estimated unigram surprisal against the corpus-estimated values via paired permutation test on item-level $\dll$ differences, with stars marking corrected significance at an $\alpha$ level of 0.05. Error bars represent 95\% confidence intervals.}
    \label{fig:rt_GAMM}
\end{figure*} 

\section{Results}
\label{sec:results}

\subsection{Marginalizing Reduces Variance}
To study empirically whether the variance of the MMC estimator is lower than that of the naive Monte Carlo estimator in \cref{eq:mc_est}, we use both to compute estimates of unigram frequency for each unigram appearing in the English portion of MECO \citep{siegelman2022expanding}. Given $\N=500$ samples, we assess the variance of an estimator $\widehat{\mu}$ via bootstrapping at subsample sizes
$\subsampsize \in \{100, 200, 300, 400, 500\}$, drawing $\bootsampsize=1000$ resamples with replacement and computing the estimator on each resample. The variance is then estimated as follows:\looseness=-1
\begin{equation}
    \label{eq:var}
    \begin{split}
    &\ \Var\!\left(\widehat{\mu}\right) = \E\!\left[(\widehat{\mu}-\E\!\left[\widehat{\mu}\right])^2\right]
    \\& = \frac{1}{\bootsampsize-1}\sum_{k=1}^\bootsampsize\!\left(\widehat{\mu_k}-\overline{\widehat\mu}\right)^2,
    \end{split}
\end{equation}
where $\overline{\widehat\mu}\defeq\frac{1}{\bootsampsize}\sum_{k=1}^{\bootsampsize}\widehat\mu_k$. Results are in \cref{tab:var} (see \cref{fig:variance_comp} for visualizations). The marginalized estimator (\cref{eq:rb_est_02}) exhibits consistently lower variance than the naive estimator (\cref{eq:mc_est}) across all subsample sizes and unigrams. Indeed, the ratio of the two variances generally grows with larger subsamples.
This confirms that marginalization is an empirically efficient method to estimate unigram probabilities. 

\begin{table}[]
\centering
\small
\begin{tabular}{lccc}
\toprule
{$J$} & {MMC variance} & {NMC variance} & {Ratio} \\
\midrule
  100 & 1.685e-08 & 1.736e-07 & 10.30 \\
  200 & 8.508e-09 & 9.689e-08 & 11.39 \\
  300 & 5.648e-09 & 6.769e-08 & 11.98 \\
  400 & 4.044e-09 & 5.448e-08 & 13.47 \\
  500 & 3.307e-09 & 4.281e-08 & 12.94 \\
\bottomrule
\end{tabular}
\caption{Mean bootstrap variance of the Marginalized (MMC) and
Naive (NMC) Monte Carlo unigram frequency estimators across subsample sizes~$J$. Ratio = NMC\,/\,MMC.}
\label{tab:var}
\end{table}

We also investigate the coverage of the naive Monte Carlo estimate, i.e., for how many unigrams in MECO it successfully obtains estimates. Obtaining an estimate with the Naive Monte Carlo estimator requires the unigram to occur in the sample pool, which is not guaranteed. Across a range of mGPT sample sizes from $N=10^2$ to $N=10^7$, we find that NMC never successfully obtains estimates for 100\% of the types in the MECO corpus for any of the 11 languages tested here, and even obtains less than 50\% coverage for Finnish types in sample pools of $10^7$ strings (see \cref{app:ap_var}).\looseness=-1

Importantly, our coverage experiment indicates that sampled text from mGPT is skewed toward Hebrew, Greek, and Russian. Thus, while the MMC estimator does not have the same coverage problems as the NMC estimator, the sampled contexts used for estimates may not be informative (i.e., because they are from different languages). To better adapt our estimator for use with multilingual language models, we introduce priority-MMC, which draws samples from a large pool with language-specific strata (see \cref{sec:ap_pps} for details).\looseness=-1

\subsection{Downstream Impact of Marginalizing}

We now focus on models of gaze duration fit with various unigram estimates: corpus-based estimates from \texttt{wordfreq} \citep{robyn_speer_2022}, NMC estimates, MMC estimates and their canonical-only approximations (MMC and MMC Approx.), and priority-MMC estimates and their canonical-only approximations (Priority MMC and Priority MMC Approx.). We present bootstrapped distributions of item-level $\dll$ for GAMMs modeling gaze duration in \cref{fig:rt_GAMM}. Across all MMC variants, we consistently observe a significant increase in $\dll$ for GAMMs fit with PMI calculated with MMC unigram surprisals, relative to GAMMs fit with PMI calculated with \texttt{wordfreq} unigram surprisals. As in prior work \citep{opedal-etal-2024-role, heilbron2023lexical}, we find that unigram surprisal is a stronger predictor of reading times than contextual surprisal; our results confirm this finding in an environment where both values come from the same underlying distribution. \looseness=-1
We note that because of relatively poor coverage (even at high sample sizes) for rarer words, more samples are excluded from the NMC models, meaning the results are likely less robust than those from the various MMC estimators. Though the MMC and priority MMC estimators properly account for non-canonical tokenizations, our findings indicate that approximations over only canonical samples yield similar results with higher coverage and less computational cost. The results highlight that principled approaches to estimating unigram surprisal have a downstream impact.\looseness=-1

\section{Conclusion}

This paper highlights an understudied matter in surprisal theory literature---while much attention has been given to how we estimate contextual surprisal, little research has focused on how we estimate unigram surprisal. We propose three principles to guide this estimation and introduce a Monte Carlo estimator of unigram surprisal from a language model, along with several marginalized variants. By deriving both surprisal quantities from the same model, differences in predictive power can be attributed to modeling decisions rather than mismatched estimation procedures.
Our experiments show that the choice of unigram surprisal estimator has empirical consequences, and we provide further support that contextual surprisal is a weaker predictor of reading times than unigram surprisal in modeling reading time. We hope this encourages future work on what modeling decisions result in cognitively plausible unigram surprisal estimations and a more careful consideration of the role of frequency in computational psycholinguistics.\looseness=-1

\section*{Limitations} Though we argue that our principles for estimating unigram surprisal are applicable to computational psycholinguistics as a whole, we limit our analyses to sentence processing. We find that the procedure for estimating unigram surprisal has an empirical impact when it comes to modeling reading times, but the claim warrants further investigation in other lines of computational psycholinguistics research, e.g., in modeling speech segmentation during language acquisition or acceptability judgment experiments.

In general, estimating unigram surprisal from a language model is computationally expensive compared to using lexical databases or corpus counts, particularly for the marginalized estimators we present. Because of their lower variance, the marginalized variants of our estimator also address this limitation of our method to some extent: relatively few samples are required to obtain stable estimates compared to the Naive Monte Carlo estimator. Still, computational resources limited our sample size and the maximum string length of the strings sampled from mGPT. 

We also study a limited scope of units. Our implementation of the unit parser matches whitespace-initial token spans (\cref{ap:boundaries}), which is adequate for the whitespace-delimited orthographies studied here but not in general. Whether the impact of the specific estimator used varies with units of interest (for example, characters or the contextually defined words studied in \citet{kiegeland-etal-2026-proper}) remains an open question.

\section*{Ethical Concerns}

This study uses eye-tracking data from human participants. Because the dataset has already been published \citep{siegelman2022expanding}, we foresee no new ethical concerns from the contributions made in this paper.

\section*{Acknowledgments}

ZH was supported by the Swiss National Science
Foundation (Grant No. 10003607).

\bibliography{latex/custom}

\appendix

\newpage

\onecolumn

\startappendixtoc

\appendixtableofcontents
\clearpage

\section{Estimator description}
\label{ap:mmc_noncanon}

\begin{itemize}
    \item \textbf{Word occurrences.} Occurrences of a word $\vocabitem$ and the split of $\counter$ into canonical and residual realizations are as in \cref{eq:can_res}, writing $C_{\vocabitem}(X) \defeq \counter(X, \vocabitem)$ and likewise for the canonical and residual parts: an occurrence at unit position $s$ of $\pi(X)$ is \emph{canonical} when the token span realizing it equals the canonical tokenization of $\vocabitem$, and \emph{residual} otherwise. We write $Z^{\mathrm{can}}_{\vocabitem,t}$ for the indicator that a canonical occurrence begins at token position $t$.
\item \textbf{Marginalize the canonical realization.} A canonical occurrence of $\vocabitem$ starting at position $t$ requires the next $k$ tokens to match its canonical tokenization, followed by a valid right boundary. We therefore define
\[
h_{\vocabitem}(a\mid x)
\defeq
\mathbb E\!\left[
Z^{\mathrm{can}}_{\vocabitem,t}
\mid
X_{<t}=x,\,
X_t=a
\right].
\]
We can then define the probability that the canonical realization of $\vocabitem$ begins at prefix $x$ and is completed with a valid right boundary as
\[
\sum_{a\in\Sigma}
\prefixLM(a\mid x)\,h_{\vocabitem}(a\mid x).
\]
Specifically, computing this quantity at a prefix $x$ requires $O((k+1)C)$ computation, where $k$ is the number of tokens in the canonical realization of $\vocabitem$ and $C$ is the cost of a single LM forward pass.
\item \textbf{Recover the exact count at EOS.} Define the per-step contribution
\[
d_{\vocabitem}(a\mid x)
=
\begin{cases}
h_{\vocabitem}(a\mid x), & a\in\Sigma,\\[5pt]
C_{\vocabitem}^{\mathrm{can}}(x)+C_{\vocabitem}^{\mathrm{res}}(x)
-\displaystyle\sum_{s=1}^{|x|} h_{\vocabitem}(x_s\mid x_{<s}),
& a=\eos.
\end{cases}
\]
Then, for every completed string $X$,
\[
\begin{aligned}
C_{\vocabitem}(X)
&=
C_{\vocabitem}^{\mathrm{can}}(X)+C_{\vocabitem}^{\mathrm{res}}(X)\\
&=
\sum_{t=1}^{|X|} h_{\vocabitem}(X_t\mid X_{<t})
+
\left[
C_{\vocabitem}^{\mathrm{can}}(X)+C_{\vocabitem}^{\mathrm{res}}(X)
-
\sum_{t=1}^{|X|} h_{\vocabitem}(X_t\mid X_{<t})
\right]\\
&=
\sum_{t=1}^{|X|}
d_{\vocabitem}(X_t\mid X_{<t})
+
d_{\vocabitem}(\eos\mid X).
\end{aligned}
\]
\item \textbf{Marginalize the per-step contribution $d_{\vocabitem}$.} At each prefix $x$, take the conditional expectation of the next-step contribution:
\begin{align*}
&\mathbb{E}_{A\sim \prefixLM(\cdot\mid x)}
\left[d_{\vocabitem}(A\mid x)\right]
= \\ &\qquad
\sum_{a\in\Sigma} \prefixLM(a\mid x)\,h_{\vocabitem}(a\mid x)
+
\prefixLM(\eos\mid x)
\left[
C_{\vocabitem}^{\mathrm{can}}(x)+C_{\vocabitem}^{\mathrm{res}}(x)
-
\sum_{s=1}^{|x|}h_{\vocabitem}(x_s\mid x_{<s})
\right].
\end{align*}

\item \textbf{Define the marginalized estimator.} For a sampled string $X$, sum the conditional expected contributions over all visited prefixes:
\[
\widehat C_{\vocabitem}^{\mathrm{MMC}}(X)
=
\sum_{t=1}^{|X|+1}
\mathbb E_{A\sim \prefixLM(\cdot\mid X_{<t})}
\left[
d_{\vocabitem}(A\mid X_{<t})
\right].
\]
\item \textbf{Show unbiasedness.} Let $\tau (X) \defeq |X|+1$.
Then
{\allowdisplaybreaks
\begin{align*}
\mathbb E_{X\sim p}
\left[
\widehat C_{\vocabitem}^{\mathrm{MMC}}(X)
\right]
&=
\mathbb E_{X\sim p}
\left[
\sum_{t=1}^{\tau(X)}
\mathbb E_{A\sim \prefixLM(\cdot\mid X_{<t})}
\left[
d_{\vocabitem}(A\mid X_{<t})
\right]
\right]
\\[4pt]
\stepnote{extend sum via indicator}&=
\mathbb E_{X\sim p}
\left[
\sum_{t=1}^{\infty}
\mathbf {1}\{|X| \ge t-1\}
\mathbb E_{A\sim \prefixLM(\cdot\mid X_{<t})}
\left[
d_{\vocabitem}(A\mid X_{<t})
\right]
\right]
\\[4pt]
\stepnote{expectation as sum}&=
\sum_{\boldsymbol{x}\in\Sigma^*}
p(\boldsymbol{x})
\sum_{t=1}^{\infty}
\mathbf {1}\{|\boldsymbol{x}| \ge t-1\}
\mathbb E_{A\sim \prefixLM(\cdot\mid \boldsymbol{x}_{<t})}
\left[
d_{\vocabitem}(A\mid \boldsymbol{x}_{<t})
\right]
\\[4pt]
\stepnote{swap sums (Fubini)}&=
\sum_{t=1}^{\infty}
\sum_{\boldsymbol{x}\in\Sigma^*}
p(\boldsymbol{x})
\mathbf {1}\{|\boldsymbol{x}| \ge t-1\}
\mathbb E_{A\sim \prefixLM(\cdot\mid \boldsymbol{x}_{<t})}
\left[
d_{\vocabitem}(A\mid \boldsymbol{x}_{<t})
\right]
\\[4pt]
\stepnote{rewrite as expectation}&=
\sum_{t=1}^{\infty}
\mathbb E_{X\sim p}
\left[
\mathbf {1}\{|X| \ge t-1\}
\mathbb E_{A\sim \prefixLM(\cdot\mid X_{<t})}
\left[
d_{\vocabitem}(A\mid X_{<t})
\right]
\right]
\\[4pt]
\stepnote{expectation as sum}&=
\sum_{t=1}^{\infty}
\sum_{\boldsymbol{x}\in\Sigma^*}
p(\boldsymbol{x})
\,
\mathbf {1}\{|\boldsymbol{x}| \ge t-1\}
\mathbb E_{A_t\sim \prefixLM(\cdot\mid \boldsymbol{x}_{<t})}
\left[
d_{\vocabitem}(A_t\mid \boldsymbol{x}_{<t})
\right]
\\[6pt]
\stepnote{split $\boldsymbol{x}=\boldsymbol{x}_{<t}\boldsymbol z$}&=
\sum_{t=1}^{\infty}
\sum_{\boldsymbol{x}_{<t} \in \Sigma^{t-1}}
\sum_{\boldsymbol z\in\Sigma^*}
\left[
\prod_{s=1}^{t-1}
\prefixLM(x_s\mid \boldsymbol{x}_{<s})
\right]\left[
\prod_{r=1}^{|\boldsymbol z|}
\prefixLM(z_r\mid \boldsymbol{x}_{<t}\boldsymbol z_{<r})
\right]
\\
&\qquad \times \prefixLM(\eos\mid \boldsymbol{x}_{<t}\boldsymbol z) \mathbf {1}\{|\boldsymbol{x}_{<t} \boldsymbol{z}| \ge t-1\} 
\mathbb E_{A_t\sim \prefixLM(\cdot\mid \boldsymbol{x}_{<t})}
\left[
d_{\vocabitem}(A_t\mid \boldsymbol{x}_{<t})
\right]
\\[6pt]
\stepnote{$\boldsymbol z$-sum equals $1$}&=
\sum_{t=1}^{\infty}
\sum_{\boldsymbol{x}_{<t} \in \Sigma^{t-1}}
\underbrace{\left[
\prod_{s=1}^{t-1}
\prefixLM(x_s\mid \boldsymbol{x}_{<s})
\right]}_{=\, \prefixLM(\boldsymbol{x}_{<t})}\\&\qquad\qquad \times
\underbrace{
\sum_{\boldsymbol z\in\Sigma^*}
\left[
\prod_{r=1}^{|\boldsymbol z|}
\prefixLM(z_r\mid \boldsymbol{x}_{<t}\boldsymbol z_{<r})
\right]
\prefixLM(\eos\mid \boldsymbol{x}_{<t}\boldsymbol z)
}_{=\,1}
\\&\qquad\qquad  \times
 \mathbf {1}\{|\boldsymbol{x}_{<t} \boldsymbol{z}| \ge t-1\}
\mathbb E_{A_t\sim \prefixLM(\cdot\mid \boldsymbol{x}_{<t})}
\left[
d_{\vocabitem}(A_t\mid \boldsymbol{x}_{<t})
\right] \\
\stepnote{indicator equals $1$}&=
\sum_{t=1}^{\infty}
\sum_{\boldsymbol{x}_{<t}\in\Sigma^{t-1}}
\prefixLM(\boldsymbol{x}_{<t})
\, \underbrace{
\mathbf{1}\{|\boldsymbol{x}_{<t}| \ge t-1\}
}_{=\,1\;\text{since }\boldsymbol{x}_{<t}\in\Sigma^{t-1}}
\mathbb E_{A_t\sim \prefixLM(\cdot\mid \boldsymbol{x}_{<t})}
\left[
d_{\vocabitem}(A_t\mid \boldsymbol{x}_{<t})
\right]
\\[6pt]
\stepnote{expand inner expectation}&=
\sum_{t=1}^{\infty}
\sum_{\boldsymbol{x}_{<t}\in\Sigma^{t-1}}
\prefixLM(\boldsymbol{x}_{<t})
\sum_{a\in\Sigma\cup\{\eos\}}
\prefixLM(a\mid \boldsymbol{x}_{<t})
\,
d_{\vocabitem}(a\mid \boldsymbol{x}_{<t})
\\[6pt]
\stepnote{split off $\eos$ term}&=
\sum_{t=1}^{\infty}
\sum_{\boldsymbol{x}_{<t}\in\Sigma^{t-1}}
\prefixLM(\boldsymbol{x}_{<t})
\Bigl[
\sum_{a\in\Sigma}
\prefixLM(a\mid \boldsymbol{x}_{<t})
d_{\vocabitem}(a\mid \boldsymbol{x}_{<t})
\Bigr.\\&\qquad\qquad\qquad\qquad\qquad\qquad\Bigl.+
\prefixLM(\eos\mid \boldsymbol{x}_{<t})
d_{\vocabitem}(\eos\mid \boldsymbol{x}_{<t})
\Bigr]
\\[6pt]
\stepnote{chain rule of $\prefixLM$}&=
\sum_{t=1}^{\infty}
\sum_{\boldsymbol{x}_{\leq t}\in\Sigma^t}
\prefixLM(\boldsymbol{x}_{\leq t})
d_{\vocabitem}(x_t\mid \boldsymbol{x}_{<t})
\\&\qquad+
\sum_{t=1}^{\infty}
\sum_{\boldsymbol{x}_{<t}\in\Sigma^{t-1}}
\prefixLM(\boldsymbol{x}_{<t})
\prefixLM(\eos\mid \boldsymbol{x}_{<t})
d_{\vocabitem}(\eos\mid \boldsymbol{x}_{<t})
\\[6pt]
\stepnote{$\eos$ term over full strings}&=
\sum_{t=1}^{\infty}
\sum_{\boldsymbol{x}_{\leq t}\in\Sigma^t}
\prefixLM(\boldsymbol{x}_{\leq t})
d_{\vocabitem}(x_t\mid \boldsymbol{x}_{<t})
+
\sum_{\boldsymbol{x}\in\Sigma^*}
p(\boldsymbol{x})
d_{\vocabitem}(\eos\mid \boldsymbol{x})
\\[6pt]
&=
\sum_{t=1}^{\infty}
\sum_{\boldsymbol{x}_{\leq t}\in\Sigma^t}
\sum_{\boldsymbol z\in\Sigma^*}
p(\boldsymbol{x}_{\leq t}\boldsymbol z)
d_{\vocabitem}(x_t\mid \boldsymbol{x}_{<t})
+
\sum_{\boldsymbol{x}\in\Sigma^*}
p(\boldsymbol{x})
d_{\vocabitem}(\eos\mid \boldsymbol{x})
\\[0.1pt]
\stepnote{regroup by $\boldsymbol{x}$}&=
\sum_{\boldsymbol{x}\in\Sigma^*}
p(\boldsymbol{x})
\sum_{t=1}^{|\boldsymbol{x}|}
d_{\vocabitem}(x_t\mid \boldsymbol{x}_{<t})
+
\sum_{\boldsymbol{x}\in\Sigma^*}
p(\boldsymbol{x})
d_{\vocabitem}(\eos\mid \boldsymbol{x})
\\[0.1pt]
\stepnote{definition of $C_{\vocabitem}$}&=\mathbb E_{X\sim p}
\left[
\sum_{t=1}^{|X|}
d_{\vocabitem}(X_t\mid X_{<t})
+
d_{\vocabitem}(\eos\mid X)
\right] = \mathbb E_{X\sim p}
\left[
C_{\vocabitem}(X)
\right].
\end{align*}}

\item \textbf{Define an expectation-equivalent contribution.}
Because the canonical correction term has zero expectation,
\[
\mathbb E_{X\sim p}
\left[
C_{\vocabitem}^{\mathrm{can}}(X)
-
\sum_{t=1}^{|X|}
h_{\vocabitem}(X_t\mid X_{<t})
\right]
=0,
\]
we can remove it from the EOS contribution without changing the expected count. We therefore define
\[
\widetilde d_{\vocabitem}(a\mid x)
\defeq
\begin{cases}
h_{\vocabitem}(a\mid x), & a\in\Sigma,\\[5pt]
C_{\vocabitem}^{\mathrm{res}}(x), & a=\eos.
\end{cases}
\]
The corresponding cumulative quantity is
\[
\widetilde C_{\vocabitem}(X)
\defeq
\sum_{t=1}^{|X|}
\widetilde d_{\vocabitem}(X_t\mid X_{<t})
+
\widetilde d_{\vocabitem}(\eos\mid X)
=
\sum_{t=1}^{|X|}
h_{\vocabitem}(X_t\mid X_{<t})
+
C_{\vocabitem}^{\mathrm{res}}(X).
\]
Unlike $d_{\vocabitem}$, the contribution $\widetilde d_{\vocabitem}$ does not recover the exact count pathwise. However,
\[
\begin{aligned}
\mathbb E_{X\sim p}
\left[
\widetilde C_{\vocabitem}(X)
\right]
&=
\mathbb E_{X\sim p}
\left[
\sum_{t=1}^{|X|}
h_{\vocabitem}(X_t\mid X_{<t})
+
C_{\vocabitem}^{\mathrm{res}}(X)
\right]
\\
&=
\mathbb E_{X\sim p}
\left[
C_{\vocabitem}^{\mathrm{can}}(X)
+
C_{\vocabitem}^{\mathrm{res}}(X)
\right]
\\
&=
\mathbb E_{X\sim p}
\left[
C_{\vocabitem}(X)
\right].
\end{aligned}
\]
Thus, $\widetilde d_{\vocabitem}$ yields an expectation-equivalent decomposition of the count, whereas $d_{\vocabitem}$ yields a pathwise-exact decomposition.
\item \textbf{Define the expectation-equivalent marginalized estimator.}

Using $\widetilde d_{\vocabitem}$, we define the corresponding marginalized estimator as
\[
\widehat C_{\vocabitem}^{\mathrm{MMC}\text{-}\widetilde d}(X)
\defeq
\sum_{t=1}^{|X|+1}
\mathbb E_{A\sim \prefixLM(\cdot\mid X_{<t})}
\left[
\widetilde d_{\vocabitem}(A\mid X_{<t})
\right].
\]
\end{itemize}

\section{MLE of Unigram Frequency}
\label{app:mle_est}
Suppose we observe strings $X^{\!\left(1\right)}, \dots, X^{\!\left(N\right)} \in \wordStrings$ sampled from some distribution over strings. A direct estimator of unigram probability is obtained by replacing the expectations in \cref{eq:freq01,eq:freq02} with empirical averages.

First, define the empirical expected count
\begin{equation}
    \hat{\mu}_N\!\left(\vocabitem\right)
    \defeq
    \frac{1}{N}\sum_{n=1}^{N} \counter\!\left(X^{\!\left(n\right)}, \vocabitem\right).
\end{equation}
Plugging this and its sum over $\vocab$ into \cref{eq:freq02} yields the estimator
\begin{equation}
    \hat{p}_N\!\left(\vocabitem\right)
    \defeq
    \frac{\hat{\mu}_N\!\left(\vocabitem\right)}
    {\sum_{\vocabitem' \in \vocab} \hat{\mu}_N\!\left(\vocabitem'\right)}.
\end{equation}

Equivalently, $\hat{p}_N\!\left(\vocabitem\right)$ is just the relative frequency of $\vocabitem$ among all word occurrences in the sample. The corresponding plug-in estimator of unigram surprisal is
\begin{equation}
    \hat{s}_{\mathrm{uni},N}\!\left(\vocabitem\right)
    \defeq
    -\log \hat{p}_N\!\left(\vocabitem\right).
\end{equation}

This is the standard corpus-based approach to unigram estimation. It is a \emph{plug-in} estimator because it estimates the population moments appearing in the definition of unigram probability and then substitutes those estimates into the target functional. Lexical resources such as frequency tables or unigram databases can be viewed as implementations of this same basic strategy, possibly with additional preprocessing or smoothing.

\begin{proposition}
\label{prop:plugin_bias}
For an unbiased estimator of unigram probability, the expected surprisal of the frequency estimator overestimates the true unigram surprisal:
\begin{equation}
    \label{eq:jensen_bias}
    \E\!\left[-\log \hat{p}\!\left(\vocabitem\right)\right]
    \;\geq\;
    -\log \freqNorm\!\left(\vocabitem\right)
    \;=\;
    \suni\!\left(\vocabitem\right).
\end{equation}
Equality holds if and only if $\hat{p}\!\left(\vocabitem\right)$ has zero variance,
i.e., is almost surely constant.
\end{proposition}

\begin{proof}
The expected surprisal of the frequency estimator is
$\E\left[-\log \hat{p}\left(\vocabitem\right)\right]$. Since $-\log$ is strictly convex,
\begin{equation}
    \E\left[-\log \hat{p}\left(\vocabitem\right)\right]
    \;\overset{\textcolor{gray}{\text{Jensen}}}{\geq}\;
    -\log \E\left[\hat{p}\left(\vocabitem\right)\right]
    \;\overset{\textcolor{gray}{\text{unbiased}}}{=}\;
    -\log \freqNorm\left(\vocabitem\right)
    = \suni\left(\vocabitem\right),
\end{equation}
where the last equality uses
$\E\left[\hat{p}\left(\vocabitem\right)\right] = \freqNorm\left(\vocabitem\right)$. Since $-\log$ is \emph{strictly} convex, the Jensen step is strict unless
$\hat{p}\left(\vocabitem\right)$ equals $\E\left[\hat{p}\left(\vocabitem\right)\right]$
almost surely. That is, unless it has zero variance, in which case the bias disappears.
\end{proof}

\section{Implementing the Unit Parser}
\label{ap:boundaries}

In practice, we do not compute $\pi$ exactly, but approximate it, and hence the count $\counter$, by matching whitespace-initial token spans. \citet{kiegeland-etal-2026-proper} argue that whitespace heuristics of this kind are not linguistically adequate in general. We adopt the heuristic as an implementation of $\pi$ for the whitespace-delimited words studied here and note this as a limitation.

In the following example, we illustrate our boundary-sensitive matching for the target word \texttt{good}. \textvisiblespace{} marks the leading whitespace carried by a token, and \textcolor{green!50!black}{green} \textvisiblespace{} symbols indicate the word boundaries relevant to the match. 

\begin{figure}[!h]
\centering
\begin{tikzpicture}[
    font=\small,
    tok/.style={draw, rounded corners=1pt, minimum height=5.5mm, inner sep=3pt,
                text height=1.8ex, text depth=0.45ex},
    ctxt/.style={draw=none, inner sep=2pt,
                text height=1.8ex, text depth=0.45ex},
    lab/.style={draw=none, inner sep=1pt, font=\scriptsize},
    every node/.style={align=center}
]

\node[lab, anchor=east] (alab) at (1.7,0) {(a) \texttt{good}};
\node[ctxt, anchor=west] (a0) at (2.0,0) {\texttt{very}};
\node[tok,  right=3mm of a0]   (a1) {\texttt{{\color{green!50!black}\textvisiblespace}go}};
\node[tok,  right=1mm of a1]   (a2) {\texttt{od}};
\node[ctxt, right=3mm of a2]   (a3) {\texttt{{\color{green!50!black}\textvisiblespace}movie}};

\draw[->] ([yshift=-4mm]a1.south) -- (a1.south) node[midway,right,lab] {$t$};

\draw[decorate,decoration={brace,amplitude=4pt}]
    ([yshift=1mm]a1.north west) -- ([yshift=1mm]a2.north east)
    node[midway,above=5pt,lab] {candidate match};

\node[lab, anchor=east] (blab) at (1.7,-1.9) {(b) \texttt{goodness}};
\node[ctxt, anchor=west] (b0) at (2.0,-1.9) {\texttt{her}};
\node[tok,  right=3mm of b0]   (b1) {\texttt{{\color{green!50!black}\textvisiblespace}go}};
\node[tok,  right=1mm of b1]   (b2) {\texttt{od}};
\node[tok,  right=1mm of b2, draw=red, text=red] (b3) {\texttt{ness}};
\node[ctxt, right=3mm of b3]   (b4) {\texttt{\textvisiblespace shows}};

\draw[->] ([yshift=-4mm]b1.south) -- (b1.south) node[midway,right,lab] {$t$};

\draw[decorate,decoration={brace,amplitude=4pt}]
    ([yshift=1mm]b1.north west) -- ([yshift=1mm]b2.north east)
    node[midway,above=5pt,lab] {candidate match};

\end{tikzpicture}
\end{figure}

In both cases the scan starts at token position $t$ on \texttt{\textvisiblespace go}, whose whitespace provides the left boundary. In (a), the span \texttt{\textvisiblespace go}\,\texttt{od} forms \texttt{good} and the whitespace of the next token \texttt{\textvisiblespace movie} provides a valid right boundary, so the occurrence is counted. In (b), the same span is followed by \textcolor{red}{\texttt{ness}}, which carries no whitespace, so the right boundary condition fails and the occurrence is not counted.

Below we enumerate the left and right boundaries the matching requires:

\begin{table}[!h]
\begin{tabular}{@{}l l p{0.52\linewidth}@{}}
\toprule
Boundary & Location & Condition (satisfied when) \\
\midrule
Left  & first token of $\vocabitem$
      & it carries a leading \textcolor{green!50!black}{\textvisiblespace}
        (word-initial marker \texttt{Ġ}), \emph{or} it begins
        the sequence \\[3pt]
Right & token following the last token of $\vocabitem$
      & it carries a leading \textcolor{green!50!black}{\textvisiblespace}
        (marker \texttt{Ġ}), \emph{or} it is \eos \\
\bottomrule
\end{tabular}
\end{table}

\section{Properties of the Naive Monte Carlo Estimator}

\newtheorem*{repproposition}{Proposition \ref{prop:mc_consistency}}
\begin{repproposition}
Assume that $\E_{\X \sim \LM}\!\left[\lvert \X \rvert\right] < \infty$ and that $\sum_{\vocabitem' \in \vocab}\mu\!\left(\vocabitem'\right) > 0$.
Then, for every fixed \(\vocabitem \in \vocab\), the Monte Carlo estimator
\(\mcest_{\N}\!\left(\vocabitem\right)\) is strongly consistent and asymptotically unbiased for
\(\freqNorm\!\left(\vocabitem\right)\). That is,
\begin{equation}
    \mcest_{\N}\!\left(\vocabitem\right)
    \xrightarrow{a.s.}
    \freqNorm\!\left(\vocabitem\right)
    \qquad \text{as } \N \to \infty,
\end{equation}
and
\begin{equation}
    \lim_{\N \to \infty}
    \E\!\left[\mcest_{\N}(\vocabitem)\right]
    =
    \freqNorm\!\left(\vocabitem\right).
\end{equation}
\end{repproposition}

\begin{proof}

Recall that $N_{\vocab}\!\left(\wordseq\right) = \sum_{\vocabitem'\in\vocab}\counter\!\left(\wordseq,\vocabitem'\right)$ denotes the number of occurrences in $\wordseq$ of words in $\vocab$, so that the estimator in \cref{eq:mc_est} can be written as
\[
    \mcest_{\N}\!\left(\vocabitem\right)
    =
    \frac{\sum_{n=1}^{\N}\counter\!\left(\X^{\!\left(n\right)},\vocabitem\right)}
         {\sum_{n=1}^{\N} N_{\vocab}\!\left(\X^{\!\left(n\right)}\right)},
\]
with the convention $\mcest_{\N}\!\left(\vocabitem\right) \defeq 0$ when the denominator is zero.
For any \(\N > 0\), define the random vector
\[
    Z_{\N}
    =
    \!\left(
        \frac{1}{\N}
        \sum_{n=1}^{\N}
        \counter\!\left(\X^{\!\left(n\right)}, \vocabitem\right),
        \frac{1}{\N}
        \sum_{n=1}^{\N}
        N_{\vocab}\!\left(\X^{(n)}\right)
    \right).
\]
Further define
$$
    \ell
    =
    \E_{\X \sim \LM}\!\left[N_{\vocab}\!\left(\X\right)\right]
    =
    \sum_{\vocabitem'\in\vocab}\mu\!\left(\vocabitem'\right).
$$
Since every occurrence of a word occupies at least one token position,
\[
    0 \leq \counter\!\left(\X,\vocabitem\right) \leq N_{\vocab}\!\left(\X\right) \leq \lvert\X\rvert,
\]
the assumption \(\E_{\X \sim \LM}\!\left[\lvert\X\rvert\right] < \infty\) implies that
\[
    \E_{\X \sim \LM}\!\left[\counter\!\left(\X,\vocabitem\right)\right] < \infty
    \qquad \text{and} \qquad
    \E_{\X \sim \LM}\!\left[N_{\vocab}\!\left(\X\right)\right] < \infty.
\]
Hence the strong law of large numbers applies to both coordinates of \(Z_{\N}\), and
\begin{equation}
    Z_{\N}
    \xrightarrow{a.s.}
    \!\left(
        \mu\!\left(\vocabitem\right),
        \ell
    \right).
\end{equation}
Now define $g \colon \mathbb{R} \times \!\left(\mathbb{R}\setminus\!\left\{0\right\}\right) \to \mathbb{R}, \qquad
g(x,y) = \frac{x}{y}.$
Since \(g\) is continuous at every point \(\!\left(x,y\right)\) with \(y \neq 0\), and since
  $\ell > 0$ by assumption
(so that, almost surely, the denominator is nonzero for all sufficiently large $\N$),
we can apply the continuous mapping theorem   \citep[Theorem~2.3]{van2000asymptotic}
\[
    \mcest_{\N}\!\left(\vocabitem\right)
    =
    g \!\left(
        \frac{1}{\N}
        \sum_{n=1}^{\N}
        \counter\!\left(\X^{\!\left(n\right)}, \vocabitem\right),
        \frac{1}{\N}
        \sum_{n=1}^{\N}
        N_{\vocab}\!\left(\X^{\!\left(n\right)}\right)
    \right)
    \xrightarrow{a.s.}
    g \!\left(
        \mu\!\left(\vocabitem\right),
        \ell
    \right) = \dfrac{ \mu\!\left(\vocabitem\right)}{ \ell} =   \freqNorm\!\left(\vocabitem\right).
\]
This proves strong consistency.
It remains to show asymptotic unbiasedness.
Since $\mcest_{\N}\!\left(\vocabitem\right) \to \freqNorm\!\left(\vocabitem\right)$ almost surely and, because $\counter\!\left(\X,\vocabitem\right) \leq N_{\vocab}\!\left(\X\right)$, we have  $    0
    \leq
    \mcest_{\N}\!\left(\vocabitem\right)
    \leq
    1$ for every $\N > 0$, by the bounded convergence theorem \citep[Theorem~1.6.7]{durrett2019probability}
\begin{equation}
    \lim_{\N \to \infty}
    \E\!\left[\mcest_{\N}(\vocabitem)\right]
    =
    \E\!\left[
        \lim_{\N \to \infty}
        \mcest_{\N}\!\left(\vocabitem\right)
    \right]
    =
    \freqNorm(\vocabitem).
\end{equation}
Thus
\begin{equation}
    \lim_{\N \to \infty}
    \!\left(
    \E\!\left[\mcest_{\N}(\vocabitem)\right]
    -
    \freqNorm\!\left(\vocabitem\right)
    \right)
    =
    0,
\end{equation}
and the estimator is asymptotically unbiased.
\end{proof}

\clearpage

\section{Supplementary Estimator Results}
\label{app:ap_var}

\subsection{Variance Analysis}

\begin{figure*}[h]
    \centering
    \includegraphics[width=0.9\linewidth]{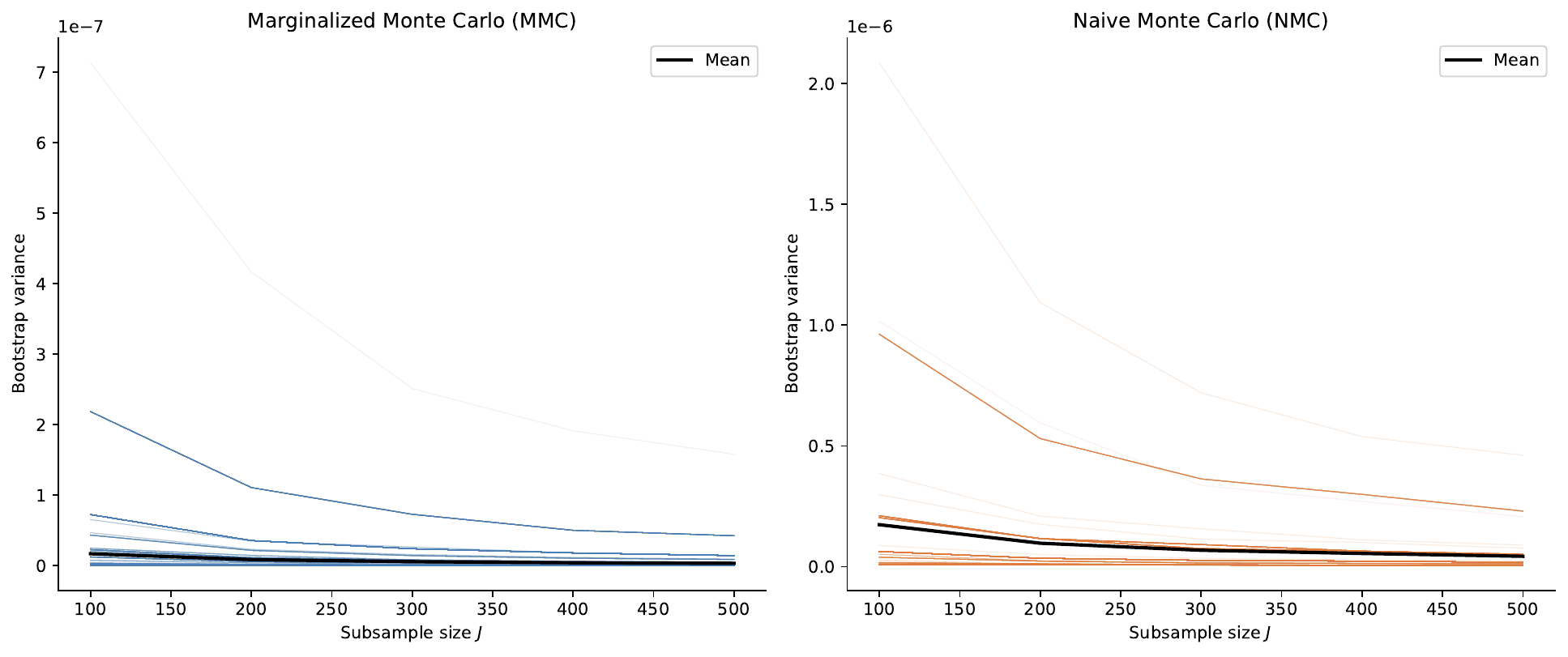}
    \caption{Visualization of the bootstrapped variance estimates of our marginalized and naive Monte Carlo estimators of unigram frequency. Each line represents a single unigram from the English MECO data. The ``Mean'' lines refer to the average variance across all unigrams.}
    \label{fig:variance_comp}
\end{figure*}

\subsection{Naive Monte Carlo Coverage}

\begin{figure*}[h]
    \centering
    \includegraphics[width=0.9\linewidth]{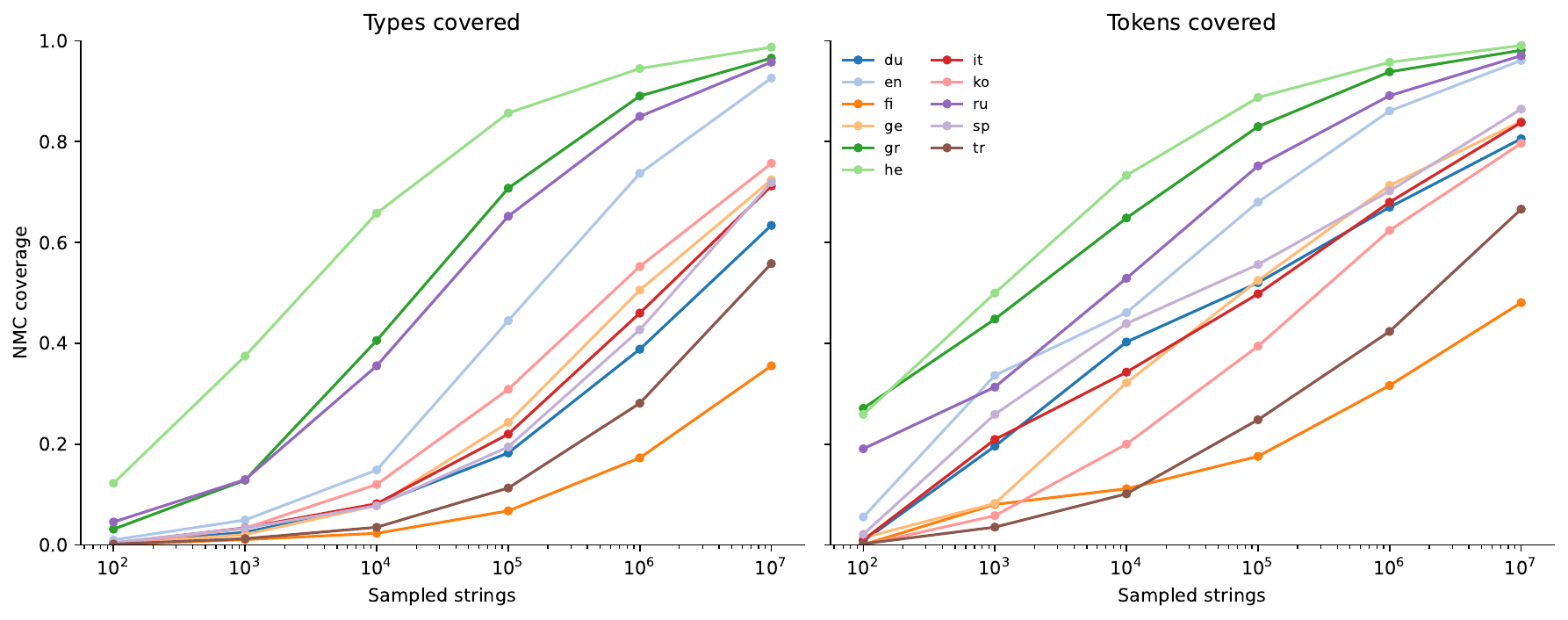}
    \caption{Proportion of types and tokens in the MECO corpus for which estimates can be obtained from our Naive Monte Carlo estimator at a range of sample sizes. Samples are obtained from mGPT with the maximum length hyperparameter set to 50 tokens. Estimates cannot be obtained when the target type/token never appears in the $N$ sampled strings.}
    \label{fig:nmc_coverage}
\end{figure*}

\subsection{Runtime Analysis}

We choose five unigrams each of lengths 1--4 tokens from the English MECO dataset and estimate their unigram probability five times each with the naive MC algorithm in \cref{eq:mc_est} and the marginalized variant in \cref{eq:rb_est_02}. Experiments are run on an NVIDIA H100 NVL GPU. 

\begin{table}[t]
\centering
\begin{tabular}{ccc}
\toprule
Tokens & NMC (s) & MMC (s) \\
\midrule
  1 & 3.788 (0.003) & 86.269 (0.101) \\
  2 & 3.788 (0.003) & 166.344 (0.161) \\
  3 & 3.787 (0.002) & 247.479 (0.105) \\
  4 & 3.783 (0.008) & 329.654 (0.153) \\
\bottomrule
\end{tabular}
\caption{Mean (SD) wall-clock time per estimate for the Naive (NMC) and Marginalized (MMC) Monte Carlo estimators,
broken down by the number of subword tokens in the target unigram ($N=500$ samples, batch size 128).}
\label{tab:runtime}
\end{table}

Our results demonstrate that for unigrams of token length $T=1$, both estimators incur similar costs, which are mainly from the generation of $N=500$ strings. For multi-token unigrams, however, the MMC estimator scales with token count while times for the NMC estimator are stable. This is a consequence of our estimation of the joint probability of all the tokens in multi-token unigrams, which requires $T-1$ additional rounds of batched forward passes over all positions in the sampled contexts. Still, this cost is paid only once per unit in the dataset, and the faster NMC estimates for multi-token words come at the cost of higher variance estimates relative to the MMC estimator (see \cref{tab:var} and \cref{fig:variance_comp}).

\section{Experiment Details}
\label{sec:ap_data}

\subsection{Reading Time Data}
\begin{table}[h]
\centering
\small
\begin{tabular}{lccccc} \toprule
Language & Participants & Word Types & Observations & Train Size & Test Size \\ \midrule
Dutch    & 45 & 662 & 31680 & 28512 & 3168 \\
English  & 46 & 930 & 55120 & 49608 & 5512 \\
Finnish  & 49 & 306 & 14853 & 13368 & 1485 \\
German   & 45 & 753 & 43942 & 39548 & 4394 \\
Greek    & 45 & 945 & 39972 & 35975 & 3997 \\
Hebrew   & 47 & 1303 & 49430 & 44487 & 4943 \\
Italian  & 54 & 791 & 48097 & 43287 & 4810 \\
Korean   & 32 & 916 & 15364 & 13828 & 1536 \\
Spanish  & 48 & 831 & 48448 & 43603 & 4845 \\
Turkish  & 29 & 528 & 12083 & 10875 & 1208 \\
Russian  & 46 & 1158 & 46641 & 41977 & 4664 \\ \bottomrule
\end{tabular}
\caption{Per-language counts for the NMC estimator: observations entering the gaze-duration model are per participant.  Train/test split refer to the size of the folds used to fit and test each of the GAMMs.}
\label{tab:descriptives_nmc}
\end{table}

\begin{table}[h]
\centering
\small
\begin{tabular}{lccccc} \toprule
Language & Participants & Word Types & Observations & Train Size & Test Size \\ \midrule
Dutch    & 45 & 996 & 44816 & 40334 & 4482 \\
English  & 46 & 945 & 55792 & 50213 & 5579 \\
Finnish  & 49 & 1042 & 49783 & 44805 & 4978 \\
German   & 45 & 1037 & 58171 & 52354 & 5817 \\
Greek    & 45 & 845 & 35558 & 32002 & 3556 \\
Hebrew   & 47 & 954 & 36248 & 32623 & 3625 \\
Italian  & 54 & 1051 & 61687 & 55518 & 6169 \\
Korean   & 32 & 1208 & 21057 & 18951 & 2106 \\
Spanish  & 48 & 1036 & 58185 & 52367 & 5818 \\
Turkish  & 29 & 1032 & 22935 & 20641 & 2294 \\
Russian  & 46 & 1133 & 45621 & 41059 & 4562 \\ \bottomrule
\end{tabular}
\caption{Per-language counts for the MMC estimator: observations entering the gaze-duration model are per participant.  Train/test split refer to the size of the folds used to fit and test each of the GAMMs.}
\label{tab:descriptives_mmc}
\end{table}

\begin{table}[h]
\centering
\small
\begin{tabular}{lccccc} \toprule
Language & Participants & Word Types & Observations & Train Size & Test Size \\ \midrule
Dutch    & 45 & 1009 & 45314 & 40783 & 4531 \\
English  & 46 & 980 & 57883 & 52095 & 5788 \\
Finnish  & 49 & 1076 & 51962 & 46766 & 5196 \\
German   & 45 & 1047 & 58861 & 52975 & 5886 \\
Greek    & 45 & 975 & 41229 & 37106 & 4123 \\
Hebrew   & 47 & 1319 & 50126 & 45113 & 5013 \\
Italian  & 54 & 1071 & 62485 & 56237 & 6248 \\
Korean   & 32 & 1251 & 21739 & 19565 & 2174 \\
Spanish  & 48 & 1045 & 58554 & 52699 & 5855 \\
Turkish  & 29 & 1056 & 23440 & 21096 & 2344 \\
Russian  & 46 & 1194 & 48066 & 43259 & 4807 \\ \bottomrule
\end{tabular}
\caption{Per-language counts for the MMC (Approx.) estimator: observations entering the gaze-duration model are per participant.  Train/test split refer to the size of the folds used to fit and test each of the GAMMs.}
\label{tab:descriptives_mmc_approx}
\end{table}

\begin{table}[h]
\centering
\small
\begin{tabular}{lccccc} \toprule
Language & Participants & Word Types & Observations & Train Size & Test Size \\ \midrule
Dutch    & 45 & 907 & 41247 & 37122 & 4125 \\
English  & 46 & 965 & 56924 & 51232 & 5692 \\
Finnish  & 49 & 823 & 39298 & 35368 & 3930 \\
German   & 45 & 950 & 53534 & 48181 & 5353 \\
Greek    & 45 & 957 & 40455 & 36409 & 4046 \\
Hebrew   & 47 & 1303 & 49454 & 44509 & 4945 \\
Italian  & 54 & 995 & 58600 & 52740 & 5860 \\
Korean   & 32 & 1118 & 19156 & 17240 & 1916 \\
Spanish  & 48 & 994 & 56012 & 50411 & 5601 \\
Turkish  & 29 & 894 & 20149 & 18134 & 2015 \\
Russian  & 46 & 1176 & 47253 & 42528 & 4725 \\ \bottomrule
\end{tabular}
\caption{Per-language counts for the Priority MMC estimator: observations entering the gaze-duration model are per participant.  Train/test split refer to the size of the folds used to fit and test each of the GAMMs.}
\label{tab:descriptives_pmmc}
\end{table}

\begin{table}[h]
\centering
\small
\begin{tabular}{lccccc} \toprule
Language & Participants & Word Types & Observations & Train Size & Test Size \\ \midrule
Dutch    & 45 & 1009 & 45314 & 40783 & 4531 \\
English  & 46 & 980 & 57883 & 52095 & 5788 \\
Finnish  & 49 & 1076 & 51962 & 46766 & 5196 \\
German   & 45 & 1044 & 58707 & 52836 & 5871 \\
Greek    & 45 & 975 & 41229 & 37106 & 4123 \\
Hebrew   & 47 & 1319 & 50126 & 45113 & 5013 \\
Italian  & 54 & 1071 & 62485 & 56237 & 6248 \\
Korean   & 32 & 1249 & 21702 & 19532 & 2170 \\
Spanish  & 48 & 1045 & 58554 & 52699 & 5855 \\
Turkish  & 29 & 1056 & 23440 & 21096 & 2344 \\
Russian  & 46 & 1194 & 48066 & 43259 & 4807 \\ \bottomrule
\end{tabular}
\caption{Per-language counts for the Priority MMC (Approx.) estimator: observations entering the gaze-duration model are per participant.  Train/test split refer to the size of the folds used to fit and test each of the GAMMs.}
\label{tab:descriptives_pmmc_approx}
\end{table}

All reading time data used in our models is obtained from the L1 portion of the MECO corpus \citep{siegelman2022expanding}. Our selection of the eleven languages used here represents the overlap between languages in MECO and those supported by the LM used in our experiments, mGPT. In preprocessing, we remove words skipped by participants and obtain first fixation duration (i.e., the length of the participant's first fixation on a word), gaze duration (the total time a participant fixates on a word before leaving it the first time) and the total fixation duration (the total time a participant fixates on a word regardless of whether they left the word and came back to it later). In their description of MECO, \citet{siegelman2022expanding} report on their procedure for aligning fixations to words and filtering data due to poor eye-tracker calibration. Tables throughout the section give the fold sizes in the 10-fold cross-validation used for GAMM evaluation (see \cref{sec:exp}), which differ between estimators due to differences in coverage.

mGPT is a 1.3B parameter decoder-only language model pretrained on data from 61 languages. \citet{shliazhko-etal-2024-mgpt} report using Wikipedia documents and the C4 dataset to pretrain mGPT. \texttt{wordfreq} estimates are fit with data from Wikipedia, subtitle and news corpora, books, web crawls, and social media. While there it is feasible that there is some overlap in the domain and content of the two resources' data sources, they are not fit on identical data.

\subsection{Model Specification and Evaluation}

As an example of how our GAMMs of reading times are structured, the formula for our surprisal-only model is as follows:
\begin{equation}
    \label{eq:gam_surp}
    \begin{split}
    \text{RT}_i = \beta_0 + f_1\!\left(s_{\text{context},i}\right) + 
    f_2\!\left(s_{\text{context},i-1}\right) \\
    + b_{\text{participant}\left(i\right)} +
    b_{\text{trial}\left(i\right)} + \epsilon_i.
    \end{split}
\end{equation}
\noindent Above, $f_1$ and $f_2$ are main-effect smooths of the surprisal of the current and previous unit, respectively. $b_{\text{participant}(i)}$ and $b_{\text{trial}(i)}$ represent random effects that account for participant and trial variation. We specifically implement main-effect smooths as cubic regression splines limited to six basis functions to avoid overfitting, as in \citet{opedal-etal-2024-role}. $\epsilon_i$ represents a Gaussian noise term.  
Predictors are standardized prior to model fitting, and we fit all GAMMs using the \texttt{bam()} function of the \texttt{mgcv} package \citep{wood_generalized_2006}.

Delta log-likelihood is calculated over held out data as follows:

\begin{align}
    \label{eq:dll}
    \dll = \frac{1}{\N}\sum_{n=1}^{\N} \Big[&\log p_{\text{model}}\left(\text{RT}_n\right) \notag \\
    &- \log p_{\text{baseline}}\!\left(\text{RT}_n\right)\Big].
\end{align}\looseness=-1

\subsection{Estimator Configurations}
\label{sec:ap_config}
 
Here we provide implementation details on the six sets of unigram surprisals compared in \cref{sec:results}. The \texttt{wordfreq} values \citep{robyn_speer_2022} are the corpus-based estimates used by \citet{opedal-etal-2024-role}. The NMC estimates (\cref{eq:mc_est}) are computed from a large pool of $\N = 4\times10^7$ samples, in order to circumvent the issue of poor coverage documented in \cref{app:ap_var}. The MMC estimates (\cref{eq:rb_est_02}) and their canonical-only approximations (\cref{eq:mmc_approx}; MMC and MMC Approx.) are fit over $\N = 500$ samples. The priority-MMC estimates and their canonical-only approximations (Priority MMC and Priority MMC Approx.) use language-specific strata drawn from a pool of $\N = 4\times10^7$ samples to balance language coverage; see \cref{sec:ap_pps}.

\newpage

\section{Variation across Languages}
\label{sec:ap_sup_exp}

Where we report $\dll$ averaged across languages (e.g., \cref{fig:rt_GAMM}), we specifically report micro-averaged $\dll$ over all of the word-level observations, pooled across the 11 languages. This means that languages with more data (see \cref{sec:ap_data}) contribute more to the means.  Throughout this section, we present the macro-averaged $\dll$ and their standard deviations for the models of gaze duration. Values are presented separately for each estimator, corresponding to each of the facets in \cref{fig:rt_GAMM}.

\begin{table}[h]
\small
\centering
\begin{tabular}{lccc} \toprule
\multicolumn{1}{l}{Unigram Surprisal Metric} & Model & Macro Averages & SD Across Languages \\ \midrule
\multirow{5}{*}{NMC Estimate} & Uni. Surp.               & 0.0265 & 0.0114 \\
& PMI                      & 0.0071 & 0.0062 \\
& PMI \& Uni. Surp.        & 0.0286 & 0.0121 \\
& Con. Surp.               & 0.0118 & 0.0049 \\
& Con. Surp. \& Uni. Surp. & 0.0283 & 0.0119 \\ \midrule
\multirow{5}{*}{Corpus Estimate} & Uni. Surp.               & 0.0282 & 0.0126 \\
& PMI                      & 0.0076 & 0.0057 \\
& PMI \& Uni. Surp.        & 0.0304 & 0.0132 \\
& Con. Surp.               & 0.0118 & 0.0049 \\
& Con. Surp. \& Uni. Surp. & 0.0299 & 0.0128 \\ \bottomrule
\end{tabular}
\caption{Word-level $\Delta$LL for GAMMs of gaze duration, macro-averaged across languages. Values correspond to the NMC estimator versus the Corpus baseline (fit on NMC's support).}
\label{tab:xlang_nmc}
\end{table}

\begin{table}[h]
\centering
\small
\begin{tabular}{lccc} \toprule
\multicolumn{1}{l}{Unigram Surprisal Metric} & Model & Macro Averages & SD Across Languages \\ \midrule
\multirow{5}{*}{MMC Estimate} & Uni. Surp.               & 0.0383 & 0.0195 \\
& PMI                      & 0.0122 & 0.0085 \\
& PMI \& Uni. Surp.        & 0.0415 & 0.0205 \\
& Con. Surp.               & 0.0211 & 0.0100 \\
& Con. Surp. \& Uni. Surp. & 0.0413 & 0.0204 \\ \midrule
\multirow{5}{*}{Corpus Estimate} & Uni. Surp.               & 0.0396 & 0.0211 \\
& PMI                      & 0.0103 & 0.0070 \\
& PMI \& Uni. Surp.        & 0.0430 & 0.0216 \\
& Con. Surp.               & 0.0211 & 0.0100 \\
& Con. Surp. \& Uni. Surp. & 0.0426 & 0.0214 \\ \bottomrule
\end{tabular}
\caption{Word-level $\Delta$LL for GAMMs of gaze duration, macro-averaged across languages. Values correspond to the MMC estimator versus the Corpus baseline (fit on MMC's support).}
\label{tab:xlang_mmc}
\end{table}

\begin{table}[h]
\centering
\small
\begin{tabular}{lccc} \toprule
\multicolumn{1}{l}{Unigram Surprisal Metric} & Model & Macro Averages & SD Across Languages \\ \midrule
\multirow{5}{*}{MMC (Approx.) Estimate} & Uni. Surp.               & 0.0391 & 0.0178 \\
& PMI                      & 0.0130 & 0.0085 \\
& PMI \& Uni. Surp.        & 0.0424 & 0.0191 \\
& Con. Surp.               & 0.0214 & 0.0105 \\
& Con. Surp. \& Uni. Surp. & 0.0421 & 0.0191 \\ \midrule
\multirow{5}{*}{Corpus Estimate} & Uni. Surp.               & 0.0396 & 0.0211 \\
& PMI                      & 0.0101 & 0.0069 \\
& PMI \& Uni. Surp.        & 0.0433 & 0.0217 \\
& Con. Surp.               & 0.0214 & 0.0105 \\
& Con. Surp. \& Uni. Surp. & 0.0429 & 0.0215 \\ \bottomrule
\end{tabular}
\caption{Word-level $\Delta$LL for GAMMs of gaze duration, macro-averaged across languages. Values correspond to the MMC (Approx.) estimator versus the Corpus baseline (fit on MMC (Approx.)'s support).}
\label{tab:xlang_mmc_approx}
\end{table}

\begin{table}[h]
\centering
\small
\begin{tabular}{lccc} \toprule
\multicolumn{1}{l}{Unigram Surprisal Metric} & Model & Macro Averages & SD Across Languages \\ \midrule
\multirow{5}{*}{Priority MMC Estimate} & Uni. Surp.               & 0.0392 & 0.0175 \\
& PMI                      & 0.0111 & 0.0078 \\
& PMI \& Uni. Surp.        & 0.0421 & 0.0184 \\
& Con. Surp.               & 0.0211 & 0.0104 \\
& Con. Surp. \& Uni. Surp. & 0.0416 & 0.0184 \\ \midrule
\multirow{5}{*}{Corpus Estimate} & Uni. Surp.               & 0.0387 & 0.0198 \\
& PMI                      & 0.0104 & 0.0069 \\
& PMI \& Uni. Surp.        & 0.0426 & 0.0204 \\
& Con. Surp.               & 0.0211 & 0.0104 \\
& Con. Surp. \& Uni. Surp. & 0.0422 & 0.0202 \\ \bottomrule
\end{tabular}
\caption{Word-level $\Delta$LL for GAMMs of gaze duration, macro-averaged across languages. Values correspond to the Priority MMC estimator versus the Corpus baseline (fit on Priority MMC's support).}
\label{tab:xlang_pmmc}
\end{table}

\begin{table}[h]
\centering
\small
\begin{tabular}{lccc} \toprule
\multicolumn{1}{l}{Unigram Surprisal Metric} & Model & Macro Averages & SD Across Languages \\ \midrule
\multirow{5}{*}{Priority MMC (Approx.) Estimate} & Uni. Surp.               & 0.0412 & 0.0201 \\
& PMI                      & 0.0110 & 0.0087 \\
& PMI \& Uni. Surp.        & 0.0436 & 0.0208 \\
& Con. Surp.               & 0.0214 & 0.0105 \\
& Con. Surp. \& Uni. Surp. & 0.0433 & 0.0207 \\ \midrule
\multirow{5}{*}{Corpus Estimate} & Uni. Surp.               & 0.0397 & 0.0211 \\
& PMI                      & 0.0101 & 0.0069 \\
& PMI \& Uni. Surp.        & 0.0433 & 0.0217 \\
& Con. Surp.               & 0.0214 & 0.0105 \\
& Con. Surp. \& Uni. Surp. & 0.0429 & 0.0216 \\ \bottomrule
\end{tabular}
\caption{Word-level $\Delta$LL for GAMMs of gaze duration, macro-averaged across languages: Values correspond to the Priority MMC (Approx.) estimator versus the Corpus baseline (fit on Priority MMC (Approx.)'s support).}
\label{tab:xlang_pmmc_approx}
\end{table}

\newpage

\section{Estimator Pseudocode}
\label{sec:ap_pseudocode}

\begin{algorithm}
\caption{Naive Monte Carlo Unigram Probability Estimator}
\begin{algorithmic}[1]
\Require Word $\vocabitem$, samples $\X_1, \ldots, \X_N$ where each $\X_n = \!\left(x_1, \ldots, x_{L_n}\right)$ is a token string
\State $\text{total Count} \gets 0$
\State $\text{total Words} \gets 0$
\For{each $n$ in range N}
    \State $\text{total Count} \gets \text{total Count} + \counter\!\left(\X_n, \vocabitem\right)$
    \State $\text{total Words} \gets \text{total Words} + \lvert\pi\!\left(\X_n\right)\rvert$
\EndFor
\State \Return $\dfrac{\text{total Count}}{\text{total Words}}$
\end{algorithmic}
\end{algorithm}

\begin{algorithm}

\caption{Marginalized Monte Carlo Unigram Probability Estimator (Canonical-Only Variant)}
\begin{algorithmic}[1]
\Require Word $\vocabitem$, samples $\X_1, \ldots, \X_N$ where each $\X_n = \!\left(x_1, \ldots, x_{L_n}\right)$ is a token string
\State $\!\left[t_1, \ldots, t_T\right] \gets \text{tokenize}\!\left(\vocabitem\right)$
\State $\text{total Prob} \gets 0$
\State $\text{total Words} \gets 0$
\For{each $n$ in range $N$}
    \State $B_n \gets \!\left\{0\right\} \cup \!\left\{b \colon x_{b+1} \text{ is unit-initial}\right\} \cup \!\left\{L_n\right\}$
    \State $\text{total Words} \gets \text{total Words} + \lvert\pi\!\left(\X_n\right)\rvert$
    \For{each boundary position $b \in B_n$}
        \State $\text{prefix} \gets \!\left(x_1, \ldots, x_b\right)$
        \State $\text{prob} \gets \prefixLM\!\left(t_1 \mid \text{prefix}\right)$
        \For{$i = 2, \ldots, T$}
            \State $\text{prob} \gets \text{prob} \cdot \prefixLM\!\left(t_i \mid \text{prefix} \circ [t_1, \ldots, t_{i-1}]\right)$
        \EndFor
        \State $\text{total Prob} \gets \text{total Prob} + \text{prob}$
    \EndFor
\EndFor
\State \Return $\dfrac{\text{total Prob}}{\text{total Words}}$
\end{algorithmic}
\label{alg:mmc}
\end{algorithm}

\newpage

\section{Priority-Sampled Evaluation at Large Sample Sizes}
\label{sec:ap_pps}

\begin{algorithm}
\caption{Priority-Sampled Marginalized Monte Carlo Estimator (canonical-only variant)}
\begin{algorithmic}[1]
\Require Word $\vocabitem$, frame $\mathcal{F}$ partitioned into strata $\!\left\{\mathcal{F}_g\right\}$, budgets $\!\left\{k_g\right\}$, pool word count $\ell_{\mathrm{pool}}$, first-token probabilities $s_b$ from the sweep
\State $\!\left[t_1, \ldots, t_T\right] \gets \text{tokenize}\!\left(\vocabitem\right)$
\If{$T = 1$}
    \State \Return $\sum_{b \in \mathcal{F}} s_b \,/\, \ell_{\mathrm{pool}}$
\EndIf
\State $\text{total} \gets 0$
\For{each stratum $g$}
    \State $q_b \gets s_b / u_b$ \textbf{ for } $b \in \mathcal{F}_g$, \textbf{ with } $u_b \sim \text{Uniform}(0,1]$
    \State $\mathcal{S}_g \gets$ the $k_g$ boundaries of $\mathcal{F}_g$ with largest $q_b$
    \State $\tau_g \gets$ the $\!\left(k_g + 1\right)$-th largest $q_b$, or $0$ if $\lvert \mathcal{F}_g \rvert \leq k_g$
    \For{each $b \in \mathcal{S}_g$}
        \State $c_b \gets 1$
        \For{$i = 2, \ldots, T$}
            \State $c_b \gets c_b \cdot \prefixLM\!\left(t_i \mid \text{prefix}_b \circ [t_1, \ldots, t_{i-1}]\right)$
        \EndFor
        \State $\text{total} \gets \text{total} + \max\!\left(s_b, \tau_g\right) \cdot c_b$
    \EndFor
\EndFor
\State \Return $\text{total} \,/\, \ell_{\mathrm{pool}}$
\end{algorithmic}
\label{alg:priority-sampled-mmc}
\end{algorithm}

\Cref{alg:mmc} evaluates $\prefixLM$ at every boundary of every sample. A unit whose probability mass lies in contexts that are rare under $\LM$ requires a large sample size $\N$ before any such context is drawn. Sampling and the sweep that returns every unit's first-token probability are paid once, no matter how many units are estimated. However, a unit of $T$ tokens requires $T-1$ forward passes at each boundary, computed separately for each unit and linear in the pool size (\cref{tab:runtime}). We therefore keep the sweep exhaustive and evaluate the continuation only at a bounded set of boundaries, chosen via priority sampling \citep{duffield2007priority}. We present two variants: a canonical-only variant (\cref{alg:priority-sampled-mmc}; Priority MMC Approx.\ in \cref{sec:results}), and a full estimator (Priority MMC), which restores the exact-count $\eos$ correction of \cref{eq:d_contrib}.

Let $\mathcal{F}$ collect every boundary position of every sample. For a word $\vocabitem$ with tokenization $\left[t_1, \ldots, t_T\right]$ we write
\begin{equation}
    s_b = \prefixLM\left(t_1 \mid \text{prefix}_b\right),
    \qquad
    c_b = \prod_{i=2}^{T} \prefixLM\left(t_i \mid \text{prefix}_b \circ [t_1, \ldots, t_{i-1}]\right),
\end{equation}
with $c_b = 1$ when $T = 1$. Summing these terms and normalizing, $\sum_{b \in \mathcal{F}} s_b c_b / \ell_{\mathrm{pool}}$, gives the canonical-only variant, writing $\ell_{\mathrm{pool}}$ for the realized word count of the pool, which approximates \cref{eq:rb_est_02}: (i) the summands are evaluated only at positions whose sampled successor is unit-initial, (ii) $c_b$ omits the right-boundary mass of \cref{ap:boundaries}, and (iii) residual occurrences are dropped (\cref{eq:mmc_approx}). The full estimator below removes these three approximations.

\begin{proposition}
\label{prop:pps_unbiased}
Fix a stratum $g$ and condition on the pool. Then
\begin{equation}
\E\left[\sum_{b \in \mathcal{S}_g} \max\!\left(s_b, \tau_g\right) c_b\right] = \sum_{b \in \mathcal{F}_g} s_b c_b
\end{equation}
\end{proposition}

\begin{proof}
If $\lvert \mathcal{F}_g \rvert \leq k_g$ then $\mathcal{S}_g = \mathcal{F}_g$ and $\tau_g = 0$,
so the two sides agree term by term. Otherwise fix $b$ and let $\tau_{-b}$ be the $k_g$-th largest
priority among $\mathcal{F}_g \setminus \{b\}$, a function of $\!\left\{u_{b'}\right\}_{b' \neq b}$
and so independent of $u_b$. Boundary $b$ is retained exactly when $q_b > \tau_{-b}$, and then the
$\!\left(k_g+1\right)$-th largest priority overall is $\tau_{-b}$, so
\begin{equation}
    \E\!\left[\max\!\left(s_b, \tau_g\right) c_b \mathbbm{1}\!\left\{b \in \mathcal{S}_g\right\} \,\middle|\, \tau_{-b}\right]
    = \max\!\left(s_b, \tau_{-b}\right) \min\!\left(1, \tfrac{s_b}{\tau_{-b}}\right) c_b
    = s_b c_b ,
\end{equation}
by $\Pr{u_b < s_b/\tau_{-b} \mid \tau_{-b}} = \min\!\left(1, s_b/\tau_{-b}\right)$ and
$\max\!\left(w, \tau\right)\min\!\left(1, w/\tau\right) = w$. 
The right-hand side does not involve $\tau_{-b}$, so averaging over it leaves $\E\!\left[\max\!\left(s_b, \tau_g\right) c_b \mathbbm{1}\!\left\{b \in \mathcal{S}_g\right\}\right] = s_b c_b$. Since
\begin{equation}
\sum_{b \in \mathcal{S}_g} \max\!\left(s_b, \tau_g\right) c_b = \sum_{b \in \mathcal{F}_g} \max\!\left(s_b, \tau_g\right) c_b\mathbbm{1}\!\left\{b \in \mathcal{S}_g\right\},
\end{equation}
summing this expectation over $b \in \mathcal{F}_g$ gives the claim by linearity of expectation.
\end{proof}

\paragraph{Stratification.} We form strata from each boundary's own prefix and refine the Latin script class with language detection. \Cref{prop:pps_unbiased} holds within each stratum, so the sum is unbiased for any allocation $\!\left(k_g\right)_g$ and the allocation may be chosen to reduce variance. Two conditions are needed: every stratum carrying mass must receive $k_g \geq 1$, since an unfunded stratum contributes nothing while its mass stays in the estimand, and the allocation must not depend on the draws, which holds because it comes from the column sums. Labels therefore enter the allocation and never a weight, so a mislabeled boundary costs precision but cannot bias the estimate.

\paragraph{Full estimator.}
The full estimator removes the three approximations in \cref{alg:priority-sampled-mmc}: 
\begin{enumerate}[label=(\roman*)]
    \item An interior position enters $\mathcal{F}$ only when its sampled successor is unit-initial, an event of probability $W_b$, the unit-initial mass of $\prefixLM\left(\cdot \mid \text{prefix}_b\right)$, so $s_b = \prefixLM\left(t_1 \mid \text{prefix}_b\right)/W_b$.
    \item The continuation $c_b$ adds the right-boundary factor of \cref{ap:boundaries}, i.e., the mass on unit-initial tokens and $\eos$ after $t_T$.
    \item Residual occurrences are recovered through $C$, which is the exact count of the unit over the pool, taken on the decoded strings so that every tokenization contributes.
\end{enumerate} 
Writing $\mathcal{T} \defeq \left\{b \in \mathcal{F} \colon \text{the sampled successor at } b \text{ is } t_1\right\}$, the full estimator is $\widehat{N}\left(\vocabitem\right) \defeq \widehat{S}\left(\vocabitem\right) + C - \widehat{H}\left(\vocabitem\right)$, where $\widehat{S}$ priority-samples $\sum_{b} s_b c_b$ over the interior boundaries and adds the exact empty-prefix term, common to all $\N$ samples, and $\widehat{H}$ estimates $\sum_{b \in \mathcal{T}} c_b$ from at most $k_{\mathrm{take}}$ positions of $\mathcal{T}$ drawn uniformly per first token, each weighted by the ratio of available to drawn positions. Since the two share the summand $c_b$, the tower property gives $\E\!\left[\widehat{S}\right] = \E\!\left[\widehat{H}\right]$, so $\widehat{N}$ is unbiased for $\N$ times the expected count of \cref{eq:freq01} under the length-truncated sampling process of \cref{sec:exp}, and the reported estimate is $\widehat{N}/\ell_{\mathrm{pool}}$. A negative realized numerator, which occurs for units far below the pool's resolution, leaves the estimate undefined and the unit excluded. Unbiasedness refers to $\widehat{N}$ before this exclusion.

\paragraph{Settings.} \Cref{tab:pps_settings} shows the sampling parameters of both variants. We allocated the per-unit budget $k$ in three steps: we divide a fixed share among the strata of the languages in which the unit occurs, half of it going to a generic Latin stratum for languages written in Latin script. Every stratum that carries mass receives a floor of $64$ boundaries, the rest is divided in proportion to each stratum's share of the unit's first-token mass, capped at the stratum's total from the first two steps. Any residue is divided by mass alone, such that the full budget is spent whenever the stored cells can serve it.

\begin{table}[h]
\centering
\begin{tabular}{lcc}
\toprule
& Approx. & Full \\
\midrule
Samples $\N$ & \multicolumn{2}{c}{$4 \times 10^{7}$} \\
Maximum length & \multicolumn{2}{c}{$50$} \\
Strata & \multicolumn{2}{c}{$21$} \\
Per-unit budget $k$ & \multicolumn{2}{c}{$8192$} \\
Stratum floor & \multicolumn{2}{c}{$64$} \\
Home share & $4096$ & $2048$ \\
Draw seed & $100000$ & $900000$ \\
Frame size $\lvert \mathcal{F} \rvert$ & $8.9 \times 10^{8}$ & $9.2 \times 10^{8}$ \\
Pool word count $\ell_{\mathrm{pool}}$ & $8.8 \times 10^{8}$ & $8.8 \times 10^{8}$ \\
Take draws $k_{\mathrm{take}}$ & n/a & $2048$ \\
Single-token units & exact & sampled \\
\bottomrule
\end{tabular}
\caption{Settings of the priority-sampled variants (Priority MMC Approx.\ and Priority MMC). Each variant draws its own pool of samples. The full variant's pool additionally stores the model's own sampled tokenizations for $\widehat{H}$. Approximate-variant single-token units need no sampling, since their numerator is the exact column sum plus the exact empty-prefix term. The full estimator's closing mass depends on the unit, so its single-token units are sampled like any other.}
\label{tab:pps_settings}
\end{table}

\end{document}